\documentclass[12pt]{article}
\usepackage{ifthen}
\usepackage{amsfonts}
\usepackage{bm}
\usepackage{hyperref}
\usepackage{fullpage}
\usepackage{amsmath, amssymb, amsthm}
\usepackage[sort&compress,square,comma,authoryear]{natbib}
\usepackage{algpseudocode}
\usepackage{graphicx}
\usepackage{times}

\usepackage{xspace}
\usepackage{verbatim}
\usepackage[]{algorithm2e}

\usepackage{xcolor}
\usepackage{thmtools}
\usepackage{thm-restate}
\usepackage{multicol}
\usepackage{multirow}
\usepackage[T1]{fontenc}
\usepackage{latexsym}
\usepackage{epstopdf}
\usepackage{bbm}
\usepackage{float}

\def\a{\alpha}
\def\b{\beta}

\def\d{\delta}

\def\eps{\ve}
\renewcommand{\epsilon}{\ve}
\def\ve{\varepsilon}

\def\k{\kappa}

\def\r{\rho}

\newcommand{\eqdef}{\mathrel{\mathop:}=}

\newcommand{\norm}[1]{\left\lVert #1 \right\rVert}

\newcommand{\onenorm}[1]{\left\lVert #1 \right\rVert_1}

\providecommand{\cX}{\mathcal{X}}
\providecommand{\cY}{\mathcal{Y}}
\providecommand{\cZ}{Z}
\providecommand{\cH}{\mathcal{H}}
\providecommand{\cF}{\mathcal{F}}

\providecommand{\VC}{\mathrm{VC}}
\providecommand{\cD}{D}
\providecommand{\cA}{\mathcal{A}}

\providecommand{\Id}{\mathrm{I}}
\providecommand{\normal}{\mathcal{N}}

\providecommand{\Rad}{\mathfrak{R}}
\providecommand{\GC}{\mathfrak{G}}
\providecommand{\eGC}{\widehat{\GC}}
\providecommand{\eRad}{\widehat{\Rad}}
\providecommand{\Com}{\mathfrak{O}}
\providecommand{\eCom}{\widehat{\Com}}

\providecommand{\EE}{\mathop{\mathbb{E}}}

\providecommand{\rv}[1]{\bm{#1}}
\newcommand{\abss}[1]{\left\lvert {#1} \right\rvert}

\DeclareMathOperator*{\argmin}{arg\,min}

\theoremstyle{plain}
\newtheorem{theorem}{Theorem}[section]
\newtheorem{lemma}[theorem]{Lemma}
\newtheorem{claim}[theorem]{Claim}
\newtheorem{informal}[theorem]{Informal Theorem}
\newtheorem{corollary}[theorem]{Corollary}

\newtheorem{definition}[theorem]{Definition}

\graphicspath{{plots/}}

\title{Learning from weakly dependent data under Dobrushin's condition}

\author {
	Yuval Dagan\thanks{Massachusetts Institute of Technology, EE\&CS, \texttt{dagan@mit.edu}}
	\and
	Constantinos Daskalakis\thanks{Massachusetts Institute of Technology, EE\&CS, \texttt{costis@csail.mit.edu}}
	\and
	Nishanth Dikkala\thanks{Massachusetts Institute of Technology, EE\&CS, \texttt{nishanthd@csail.mit.edu}}
	\and
	Siddhartha Jayanti\thanks{Massachusetts Institute of Technology, EE\&CS, \texttt{jayanti@mit.edu}}
}

\begin{document}
\maketitle
\begin{abstract} Statistical learning theory has largely focused on learning and generalization given independent and identically distributed (i.i.d.) samples. Motivated by applications involving time-series data, there has been a growing literature on learning and generalization in settings where data is sampled from an ergodic process. This work has also developed complexity measures, which appropriately extend the notion of Rademacher complexity to bound the generalization error and learning rates of hypothesis classes in this setting. Rather than time-series data, our work is  motivated  by settings where data  is sampled on a network or a spatial domain, and thus do not fit well within the framework of prior work. We provide learning and generalization bounds for data that are complexly dependent, yet their distribution satisfies the standard Dobrushin's condition. Indeed, we show that the standard complexity measures of Gaussian and Rademacher complexities and VC dimension are sufficient measures of complexity for the purposes of bounding the generalization error and learning rates of hypothesis classes in our setting. Moreover, our generalization bounds only degrade by constant factors compared to their i.i.d.~analogs, and our learnability bounds degrade by log factors in the size of the training set.
\end{abstract}
\section{Introduction} \label{sec:intro}

A main goal in statistical learning theory is understanding whether observations of some phenomenon of interest can be used to make confident predictions about future observations. Usually this question is studied in the setting where a training set $\rv{S}=(\rv{x}_i,\rv{y}_i)_{i=1}^m$, comprising pairs of covariate vectors $\rv{x}_i \in {\cal X}$ and response variables $\rv{y}_i \in {\cal Y}$, are drawn independently from some unknown distribution $D$, and the goal is to make predictions about a future sample $(\rv{x},\rv{y})$ drawn independently from the same distribution $D$. That is, we wish to predict $\rv{y}$ given $\rv{x}$.

Given some hypothesis class ${\cal H} \subset {\cal Y}^{\cal X}$, comprising predictors that map $\cal X$ to ${\cal Y}$ and a loss function $\ell: {\cal Y}^2 \rightarrow \mathbb{R}$ whose values $\ell(\hat{y},y)$ express how bad it is to predict $\hat{y}$ instead of $y$, a wealth of results characterize the relationship between the size  $m$ of the training set $\rv{S}$ and the approximation accuracy that is attainable for choosing some predictor $h \in {\cal H}$ whose expected loss, $L_D(h)=\mathbb{E}_{(\rv{x},\rv{y}) \sim D} \ell(h(\rv{x}),\rv{y})$, on a future sample, is as small as possible. A related question is understanding how well the training set $\rv{S}$ ``generalizes,''  in the sense of minimizing $\sup_{h \in {\cal H}}|L_S(h) - L_D(h)|$, where $L_S(h)$ is the average loss of $h$ on the training set $\rv{S}$. To characterize the learnability and generalization properties of hypotheses classes, standard complexity measures of function classes, such as the VC dimension~\citep{vapnik2015uniform} and the Rademacher complexity~\citep{bartlett2002rademacher}, have been developed.

\smallskip The assumption that the training examples $(\rv{x}_1,\rv{y}_1),\ldots,(\rv{x}_n,\rv{y}_n)$ as well as the future test sample $(\rv{x},\rv{y})$ are all independently and identically distributed (i.i.d.) is, however, too strong in many applications. Often, training data points are observed on nodes of a network, or some spatial or temporal domain, and are {\em dependent} both with respect to each other and with respect to  future observations. Examples abound in financial and meteorological applications, and dependencies naturally arise in social networks through {\em peer effects}, whose study has recently exploded in topics as diverse as criminal activity (see e.g.~\citealp{glaeser1996crime}), welfare participation (see e.g.~\citealp{bertrand2000network}), school achievement (see e.g.~\citealp{sacerdote2001peer}), participation in retirement plans (see~\citealp{duflo2003role}), and obesity (see e.g.~\citealp{trogdon2008peer,christakis2013social}). A prominent dataset where network effects are studied was collected by the National Longitudinal Study of Adolescent Health, a.k.a.~AddHealth study~\citep{harris2009waves}. This was a major national study of students in grades 7-12, who were asked to name their friends---up to 10, so that friendship networks can be constructed, and answer hundreds of questions about their personal and school life, and it also recorded information such as the age, gender, race, socio-economic background, and health of the students. Disentangling individual effects from network effects in such settings is a recognized challenge (see e.g.~the discussion by \citealt{manski1993identification} and \citealt{bramoulle2009identification}, and the discussion of prediction models for network-linked data by \citealt{li2016prediction}).

\smallskip Motivated by such applications, a growing literature has studied learning and generalization in settings where data is non-i.i.d. This work goes back to at least~\cite{yu1994rates}, and has grown quite significantly in the past decade. A central motivation  has been settings involving time-series data. As such, this literature has focused on  data sampled from an ergodic process. For this type of data, generalization and learnability bounds have been obtained whose quality depends on the mixing properties of the data generation process as well as the complexity of the hypothesis class under consideration, through appropriate generalizations of the Rademacher complexity. We discuss this literature in Section~\ref{sec:related}, and present precise generalization bounds derived from this literature in Section~\ref{sec:timeseries-comparison}. 

\smallskip In contrast to prior work, our main motivation is the study of networked data, due to their significance in  economy and society, including in the applications discussed above. The starting point of our investigation is that data observed on a network does not fit well the statistical learning frameworks proposed for non-i.i.d.~data in prior work, which targets time-series data. In particular, there is no natural ordering of  observations collected on a network with respect to which one may postulate a fast-mixing/correlation-decay property, which may be exploited for statistical power. We thus propose a different statistical learning framework that is better suited to networked data. 

We propose to study generalization and learnability when the training samples $\rv{S}=(\rv{x}_i,\rv{y}_i)_{i=1}^m$ are complexly dependent but their joint distribution satisfies Dobrushin's condition; see Definition~\ref{def:dob}. Dobrushin's condition was introduced by~\cite{dobruschin1968description} in the study of Gibbs measures, originally in the context of identifying conditions under which the Gibbs distribution has a unique equilibrium / stationary state and has since been well-studied in statistical physics and probability literature (see e.g.~\citealp{dobrushin1987completely,stroock1992logarithmic}) as it implies a number of desirable properties, such as fast mixing of Glauber dynamics \citep{kulske2003concentration}, concentration of measure \citep{marton1996bounding,kulske2003concentration,chatterjee2005concentration,DaskalakisDK18,GheissariLP17}, and correlation decay \citep{kunsch1982decay}. 
For a survey of  properties  resulting from Dobrushin's condition see \cite{weitz2005combinatorial}.

\subsection{Our Results}
\noindent \textbf{Setting:} Assuming that our training set $\rv{S}$ and test sample $(\rv{x},\rv{y})$ are drawn from a distribution $D^{(m)}$ satisfying Dobrushin's condition, as described above, we establish a number of learnability and generalization results. We make the assumption that every example in our training set $(\rv{x}_i,\rv{y}_i)$ comes from the same marginal distribution $D$ which is also the distribution from which we draw the test sample. This assumption is made to provide a uniform benchmark to measure the performance of our learning algorithms against. \\

Our first main result, presented as Theorem~\ref{thm:pac-learn}, provides an agnostic learnability bound for any hypothesis class that is learnable in the i.i.d.~setting, for instance, classes of finite VC dimension (Corollary~\ref{cor:vc-learn}). The dependence on the error and confidence in our bounds match those of the i.i.d.~setting up to logarithmic factors in the size of the training set. We focus on hypothesis classes which have small sample compression schemes in our paper. This property is known to imply learnability in the i.i.d. setting \cite{} and we will show learnability for these classes under Dobrushin distributed data as well. We provide an informal statement of our learnability result applied to finite VC dimension hypothesis classes here.
\begin{informal}[Learnability under Dobrushin Dependent Data]
	\label{thm:learnability-informal}
	Let $\cH$ be a hypothesis class such that $VC(\cH) = d$, and let $L_D$ be the expected 0/1 loss function evaluated on a sample from $D$. Given a training sample $\rv{S} \sim \cD^{(m)}$ where $\cD^{(m)}$ satisfies Dobrushin's condition, there exists a learning algorithm $\cA$ such that
	\begin{align*}
	\Pr\left[ L_D( \cA(\rv{S}) ) \le \inf_{h \in \cH} L_D(h) +  \epsilon \right] \ge 99/100, \text{ for } m = \widetilde{O}\left(\frac{d}{\epsilon^2}\right).
	\end{align*}
	
\end{informal}

Our second main result, presented as Theorem~\ref{thm:main}, provides a generalization bound for hypothesis classes, under stronger conditions on the distribution of $\rv{S}$, which we term \emph{bounded log-coefficient}, and define in Section~\ref{sec:uc}. We bound the maximal deviation $\sup_{h \in \cH} |L_D(h) - L_S(h)|$ in terms of the Gaussian complexity of $\cH$, a value which is closely related to the Rademacher complexity. We obtain a bound which is nearly as tight as if the training set $\rv{S}$ was drawn i.i.d.

\begin{informal}[Uniform Convergence under High Temperature Data]
	\label{thm:uc-informal}
		Let $\cH$ be a hypothesis class, and let $L_D(h)$ and $L_S(h)$ denote the training and expected loss, respectively, of a hypothesis $h$ with respect to some arbitrary loss function. Given a training sample $\rv{S} \sim \cD^{(m)}$ where $\cD^{(m)}$ has log-coefficient bounded by $1$, the following holds:
	\begin{align} \label{eq:1742}
	\EE \left[\sup_{h \in \cH} |L_D( h ) - L_S(h)|\right] \le O(\GC_{\cD^{(m)}}(\cH)),
	\end{align}
	where $\GC_{\cD^{(m)}}(\cH)$ is the Gaussian complexity of $\cH$. In particular, if $\VC(\cH) = d$, then the left hand side of \eqref{eq:1742} is bounded by $O(\sqrt{d/n})$.
\end{informal}

\subsection{Organization}
In Section~\ref{sec:related} we discuss the related studies along the direction of non i.i.d. generalization and learnability. 
In Section~\ref{sec:prelim} we state some preliminary notation and definitions and some lemmas from prior work that we use throughout the paper.
Section~\ref{sec:motivation} contains motivating examples for learning from data that satisfies Dobrushin's condition.
Section~\ref{sec:learnability} contains the learnability results for data satisfying Dobrushin's condition.
Section~\ref{sec:uc} contains the uniform convergence bound for data satisfying the stronger condition and the proofs for this section appear in Section~\ref{sec:pr-uc}.
Section~\ref{sec:timeseries-comparison} provides a comparison between our proposed framework and that of prior work on ergodic processes, and the benefits from our framework in terms of sharpness of generalization bounds. In particular, we show an example setting where our bounds are a significant improvement over the bounds implied by prior work.

\subsection{Related Work}
\label{sec:related}

Rademacher and Gaussian compexities for obtaining uniform convergence bounds on generalization of learning algorithms were first introduced in the work of \cite{bartlett2002rademacher} and have since been extensively studied in the literature on learning theory to characterize the sample complexity of learning for a wide range of problems.
Extending them beyond i.i.d. settings was mainly studied in the context of ergodic processes and exchangeable sequences. 
The bounds in the literature on ergodic processes typically depend on the $\a$ or $\b$ mixing coefficients of these processes. The work on studying learnability for stationary mixing empirical processes started with seminal work of \cite{yu1994rates} and was continued by \cite{mohri2009rademacher, kuznetsov2015learning, mohri2010stability} and the references therein. 
\cite{kuznetsov2015learning} studies non-stationary and non-mixing time series,
\cite{kuznetsov2014} and \cite{Kuznetsov2017} study non-stationary and mixing time series,
\cite{McDonald2017} studies stationary and non-mixing time series, and
\cite{mohri2009rademacher} studies stationary mixing time series.
Of these works, \cite{mohri2009rademacher} is most relevant to ours, since \cite{McDonald2017}'s work on non-mixing time series solves the forecasting problem, i.e. predicting $\rv{z}_{m+1}$  given previous data $\{ \rv{z}_i \}_{i=1}^m$, rather than on predicting $\rv{y}_{m+1}$ given $\rv{x}_{m+1}$ as in our setting. Moreover, since our work focuses on distributions with identical marginals, the most closely related time-series setting to ours is one where the series is stationary. Hence, we compare our results to previous work on stationary time series in Section~\ref{sec:timeseries-comparison}.

\cite{agarwal2013generalization} study the generalization properties of online algorithms in the context of stationary and mixing time series. 
Another direction in which dependent data have been considered is the setting of exchangeable sequences studied in the works of \cite{berti2009rate,pestov2010predictive} and references therein.
Apart from the above extensions to non i.i.d. data, notions of sequential Rademacher complexity were considered in the literature on online learning (see \cite{rakhlin2010online}). None of these settings capture the type of dependences we handle in our work which can have long-range correlations and no spatial mixing behavior in general.


\section{Preliminaries}
\label{sec:prelim}

\paragraph{Notational Conventions}
Random variables will be written in a bold font (say $\rv{x}$), as opposed to elements from the domain set, which are in a normal font (say, $x$). We will use the notation $C, C', C_1, c, c'$ etc. to denote positive universal constants without explicitly stating it. Given a vector $x = (x_1,\dots,x_m)$ and $i \in \{1,\dots, m\}$, $x_{-i}$ denotes the vector $x$ after omitting coordinate $i$. Given a random variables $\rv{z}, \rv{w}$ over $(\Omega, \mathcal{F})$ and $z,w \in \Omega$, denote by $P_{\rv{z}}(z)$ the probability that $\rv{z} = z$ if $\rv{z}$ is discrete and the density of $\rv{z}$ at $z$ if $\rv{z}$ is continuous. Additionally, define by $P_{\rv{z} \mid \rv{w}}(z \mid w)$ the probability $\Pr[\rv{z} = z \mid \rv{w} = w]$ if $\rv{z}$ and $\rv{w}$ are discrete and analogously if they are continuous.\footnote{For general random variables, one can define $P_{\rv{z}} = d\mu$ where $\rv{z} \sim \mu$, however, we will ignore this here. Additionally, we will assume that the density is properly defined on all the space (rather than being defined almost everywhere. Also, we assume that the conditional distributions are properly defined.}

\subsection{Learning}

Fix some feature set $\cX$, label set $\cY$, and a class of hypotheses $\cH$, containing functions from $\cX$ to $\cY$. Assume a loss function $\ell \colon \cY^2 \to \mathbb{R}$, where $\ell(\hat{y} , y)$ is the loss of predicting $\hat{y}$ when the true label is $y$. The simplest example of a loss function is the 0-1 loss, $\ell^{01}(\hat{y},y) = \mathbbm{1}_{\hat{y} \ne y}$. For any hypothesis $h \in \cH$, one can define the loss function $\ell_h \colon (\cX \times \cY) \to \mathbb{R}$ by taking $\ell_h(x,y) = \ell(h(x),y)$. Given some distribution $D$ over $\cX \times \cY$, one can define the expected loss of $h$, namely, $L_D(h) := \EE_{(x,y)\sim D} \ell_h(x,y)$. 

Let $\rv{S} = (\rv{s}_1, \dots, \rv{s}_m)\in (\cX\times \cY)^m$ be a training set of $m$ examples. Usually the coordinates of $\rv{S}$ are assumed independent and identically distributed (iid) according to $D$, but we consider more general measures; this will be discussed shortly. The goal of a learning algorithm is to choose a hypothesis $\hat{h} \in \cH$ given a sample $\rv{S}$ to (approximately) minimize the test error, $L_{D}(\hat{h})$. A common approach for doing so is taking the \emph{empirical risk minimizer} (ERM), namely, 
\[
\hat{h}_{\mathrm{ERM}}:= \arg\min_{h \in \cH} L_{\rv{S}}(h) ~; \quad
\text{where } L_S(h) = \frac{1}{m} \sum_{i=1}^m \ell_h(s_i).
\]

If the sample $\rv{S}$ ``represents well'' the test distribution $D$, then the learned hypothesis only suffers low error. To be precise, we say that $S$ is {\em $\varepsilon$-representative} if for all $h \in \cH$, $|L_D(h) - L_{S}(h)| \le \varepsilon$. From the triangle inequality, it follows that if $S$ is $\varepsilon$-representative, then the ERM is $2\varepsilon$-optimal with respect to $\cH$, namely
\[
L_D(\hat{h}_{\mathrm{ERM}}) \le \inf_{h \in \cH} L_D(h) + 2 \varepsilon.
\]

Thus, to prove learnability, it suffices to show that $\rv{S}$ is $\varepsilon$-representative. Note that representativeness is stronger than learnability: it implies that \emph{any} algorithm {\em generalizes}, namely, that the difference between the training and test errors, $L_D(\cdot)$ and $L_{\rv{S}}(\cdot)$, is  small point-wise.

\paragraph{Learning from dependent samples.}
Instead of assuming that the samples are iid, we assume that they are drawn from a \emph{dependent} (joint) distribution $D^{(m)}$ over $(\cX \times \cY)^m$, where all marginals are distributed according to the same distribution $D$ over $\cX\times \cY$. Given $\rv{S} \sim D^{(m)}$, the goal is to (approximately) minimize the test error $L_D(\hat{h})$.



\paragraph{Rademacher, Gaussian and $\rv{\tau}$ complexities.}

Given a sample $S = (s_1, \dots, s_m) \in \cZ^m$, a family $\cF$ of functions from $\cZ$ to $\mathbb{R}$, and a random variable $\rv{\tau} = (\rv{\tau}_1, \dots, \rv{\tau}_m)$ over $\mathbb{R}^m$, define the 
\emph{$\rv{\tau}$-complexity} of $\cF$ with respect to the sample $S$ by:
\[
\eCom_{S}^{\rv{\tau}}(\cF)
= \EE_{\rv{\tau}} \left[ \sup_{f \in \cF}
\frac{1}{m}\sum_{i=1}^m \rv{\tau}_i f(s_i)
\right].
\]
Define the \emph{Rademacher complexity} of $\cF$ by $\eRad_S(\cF) := \eCom^{\rv{\sigma}}_S(\cF)$ where $\rv{\sigma}$ is uniform over $\{-1,1\}^m$, and the \emph{Gaussian complexity} of $\cF$ by $\eGC_S(\cF) = \eCom^{\rv{g}}_S(\cF)$ where $\rv{g} \sim \mathcal{N}(0, I_m)$. Given a distribution $D^{(m)}$ over $\mathbb{R}^m$, define $\Com^{\rv{\tau}}_{D^{(m)}}(\cF) = \EE_{\rv{S} \sim D^{(m)}} \left[\eCom^{\rv{\sigma}}_{\rv{S}}(\cF)\right]$, and similarly define $\Rad_{D^{(m)}}(\cF)$ and $\GC_{D^{(m)}}(\cF)$.

\subsection{Weakly dependent distributions}

We define two conditions classifying weakly dependent distributions: Dobrushin's condition and high temperature in Markov Random Fields, the first being the weakest and the last being the strongest.

\subsubsection{Dobrushin's Condition \citep{dobruschin1968description}}

First, one defines influences between coordinates of a random variable $\rv{z} = (\rv{z}_1, \dots, \rv{z}_m)$. The influence from $\rv{z}_i$ to $\rv{z}_j$ captures how strong the value of $\rv{z}_j$ affects the conditional distribution of $\rv{z}_i$ when all other coordinates are fixed. Formally:
\begin{definition}[Influence in high dimensional distributions]
	\label{def:influence}
	Let $\rv{z} = (\rv{z}_1, \dots, \rv{z}_m)$ be a random variable over $Z^m$. For $i\ne j \in \{1,\dots,m\}$, define the influence of variable $\rv{z}_j$ on variable $\rv{z}_i$ as
	\begin{equation*}
	I_{j \to i}(\rv{z}) = \max_{\substack{z_{-i-j} \in Z^{m-2} \\ z_j , z_j' \in Z}} 
	d_{TV}\left( P_{\rv{z}_i \mid \rv{z}_{-i}}(\cdot \mid z_{-i-j} z_j),~ P_{\rv{z}_i \mid \rv{z}_{-i}}(\cdot \mid z_{-i-j} z_j') \right),
	\end{equation*}
	where $d_{TV}$ denotes the total variation distance.
\end{definition}

Dobrushin's condition as defined next, certifies that a weakly dependent random vector behaves as i.i.d with respect to some important properties.

\begin{definition}[Dobrushin's Uniqueness Condition]
	\label{def:dob}
	Consider a random variable $\rv{z}$ over $Z^m$. Define the \emph{Dobrushin coefficient} of $\rv{z}$ as
	$
	\a\left(\rv{z}\right) = \max_{1 \le i \le m} \sum_{j \ne i} I_{j\to i}(\rv{z}).
	$
	The variable $\rv{z}$ is said to satisfy Dobrushin's uniqueness condition if $\a\left(\rv{z}\right) < 1$.
\end{definition}

Note that the constant $1$ is important, and for $\epsilon > 0$ there are examples of vectors which deviate from the bound by $\epsilon$ and are extremely dependent.
Distributions satisfying the above condition satisfy McDiarmid-like inequalities and are $O(1/(1-\a))$-subGaussians, as presented next.

The following result builds upon the seminal studies on concentration of measure phenomenon for contracting Markov chains by \cite{marton1996bounding} which is one of the first results on concentration of measure for non-product, non-Haar measures. Theorem~\ref{thm:dob-conc} is from \cite{kulske2003concentration} and \cite{Chatterjee05}.
\begin{theorem}[Concentration of Measure under Dobrushin's Condition]
	\label{thm:dob-conc}
	Let $P^{(m)}$ be a distribution defined over $Z^m$ satisfying Dobrushin's condition with coefficient $\a$. Let $\rv{z} = (\rv{z}_1,\ldots,\rv{z}_m) \sim P^{(m)}$ and let $f: Z^m \to \mathbb{R}$ be a real-valued function with the following bounded differences property, with parameters $\lambda_1, \dots,\lambda_m \ge 0$:
	\[
	\forall z, z' \in Z^m \colon \quad
	|f(z) - f(z')| \le \sum_{i=1}^m \mathbbm{1}_{z_i \ne z'_i} \lambda_i.
	\]
	Then, for all $t > 0$,
	\begin{align*}
		\Pr\left[\abss{f(\rv{z}) - \EE[f(\rv{z})]} \ge t \right] \le 2\exp\left(-\frac{(1-\a)t^2}{2\sum_{i=1}^m \lambda_i^2}\right).
	\end{align*}
\end{theorem}

\subsubsection{Markov random fields (MRFs) with pairwise potentials}

A common way to define a random vector is by a Markov Random Field (MRF). They are defined by potential functions, which are define the correlations between the vector entries. We will be using the definition of an MRF with pairwise potentials, as defined below:

\begin{definition}[Markov Random Field (MRF) with pairwise potentials]
The random vector $\rv{z} = (\rv{z}_1, \dots, \rv{z}_m)$ over $Z^m$ is an MRF with pairwise potentials if there exist functions $\varphi_i \colon Z \to \mathbb{R}$  and $\psi_{ij} \colon \cZ^2 \to \mathbb{R}$ for $i \ne j \in \{1,\dots,m\}$ such that for all $z \in \cZ^m$,
\[
\Pr_{\rv{z} \sim P^{(m)}}[\rv{z} = z] = \prod_{i=1}^m e^{\varphi_i(z_i)}
\prod_{1 \le i < j \le m} e^{\psi_{ij}(z_i, z_j)}.
\]
The functions $\varphi_i$ are called as \emph{element-wise potentials} and $\psi_{ij}$ are \emph{pairwise potentials}.
\end{definition}

Analogous to Dobrushin's coefficient, one can define the inverse temperature of an MRF with pairwise potentials, where low inverse temperature implies weak correlations.

\begin{definition}[High Temperature MRFs]
	\label{def:high-temp}
	Given an MRF $\rv{z}$ with potentials $\{\varphi_i\}$ and $\{\psi_{ij}\}$, define
	\[
	\beta_{i,j}(\rv{z}) = \sup_{z_i, z_j \in Z} |\psi_{ij}(z_i z_j)| ~;\quad
	\beta(\rv{z}) = \max_{1 \le i \le m} \sum_{j \ne i} \beta_{ij}(P^{(m)}).
	\]
	We say that $\rv{z}$ is \emph{high temperature} if the inverse temperature, $\rv{z}$, is less than $1$.
\end{definition}
The inverse temperature is bounded by Dobrushin's coefficient, as presented below.
The proof is a simple calculation that can be found in \cite{Chatterjee05} after the statement of Theorem 3.8.
\begin{lemma}
	\label{lem:hightemp-implies-dobrushin}
	Given an MRF $\rv{z}$ with pairwise potentials, for any $i \ne j$, $I_{j \to i}(\rv{z}) \le \beta_{j,i}(\rv{z})$. Hence, $\alpha(\rv{z}) \le \beta(\rv{z})$.
\end{lemma}
Lemma~\ref{lem:hightemp-implies-dobrushin} implies that if the inverse temperature is less than $1$, then the random variable has i.i.d-like properties. Similarly to the case with Dobrushin's condition, the smallest excess in the inverse temperature over the threshold of $1$ may cause the vector to be extremely correlated.

\section{Motivation and examples}
\label{sec:motivation}


In this section, we present some tangible networked data models that would benefit from the learnability results that we prove.
Consider the problem of predicting which of many possible choices a person in a social network will make:
who will she vote for in a presidential election? what brand of smart phone will he buy? or what major will she study in college?
Each individual's choice would, of course, be dependent on her own features;
but realistically it would also depend on the choices of her friends and acquaintances.
These situations are well studied, and often modeled as opinion dynamics \citep{montanari2010}, and as autoregressive models \citep{sacerdote2001peer}. 
The following is a natural opinion dynamics for a binary decision (for instance, voting Democrat versus Republican), given the graph of a friend network:
\begin{enumerate}
\item
each individual starts with an initial preference (Democrat or Republican)
\item
and at each time step, a random individual stochastically updates his preference conditioned on the current preferences of his friend
\end{enumerate}
The stationary distribution of these dynamics is a pairwise graphical model,
and thus our learnability and generalization results would apply to the model if the influences are not too high. 

Another prominent econometric model for peer effects hinges on autoregression.
The standard linear regressive stochastic process is described by the equation $Y = XB + E$.
This equation, states that the $n \times d$ feature matrix $X$ of $n$ samples with $d$ features each has a linear relationship (given by $B$) to the response variables coded as the $n \times 1$ response vector $Y$,
up to a small stochastic error $E$ (dimension $n \times 1$).
The linear {\em autoregressive} process is described by the equation  $Y = XB + E + AY$, where the vector of responses $Y$ appears on both sides of the equation.
Thus, the model ultimately states that the responses are dependent linearly on the features and other responses.
Given a training set satisfying these equations, the task is to predict the response $y_{n+1}$ for a new sample $x_{n+1}$.
This problem falls under our model when the auto-regression matrix $A$ has entries that are sufficiently small.
Sacerdote uses this type of autoregressive model to analyze relationships between college roommate assignment and academic achievement \citep{sacerdote2001peer}.
Our results can potentially help in predictive analysis for the same question.

Our results can potentially be applied to meteorological sensor data, since values from sensors that are geographically close are likely to be correlated. 
In contexts where the influences are observed to be small enough, there is scope to leverage our results.
The AddHealth example cited in the introduction provides a setting of networked data where if certain covariates are weakly correlated across students, then one can obtain generalization bounds using our theory for prediction and regression tasks.

\section{Agnostic Learnability of Dobrushin Dependent Data}
\label{sec:learnability}

In this Section, we study learnability under data which is weakly dependent. 
Qualitatively, learnability implies the existence of a learning algorithm (not necessarily efficient) whose error goes to 0 with confidence going to 1 as the sample size increases. An algorithm $\cA$ can be shown to be a learner by showing two properties: (a) Given a training set $\rv{S}$, $\cA$ achieves a small training error on the training set, i.e. $L_{\rv{S}}(\cA(\rv{S})) \le \inf_{h \in \cH} L_{\rv{S}}(h) + O(\epsilon)$ and (b) the hypothesis output by $\cA$ generalizes, i.e. 
$\abss{L_{\rv{S}}(\cA(\rv{S})) - L_D(\cA(\rv{S}))} \le O(\epsilon)$ where $\lim_{m \to \infty} \epsilon = 0$. Here we show that we can achieve the same rates of convergence (up to log factors) for the error of the learning algorithm and the confidence bounds as in the i.i.d. setting.
For simplifying the exposition of our proof, we focus on the setting where our loss function is 0/1. Our learnability result can be extended to more general loss functions as well using techniques described in Section~\ref{sec:general-loss}.\\

We first characterize certian properties of the joint distribution of our samples which we show are sufficient to achieve learnability. Then we show that these properties hold under Dobrushin's condition (\ref{def:dob}). 
For binary hypothesis classes, in the i.i.d. setting, it is well-known that having a finite VC-dimension is equivalent to learnability. For more general hypothesis classes under the 0/1 loss function, \cite{moran2016sample} show that learnability is equivalent to having a finite size sample compression scheme. In our work, we employ the technique of sample compression to understand learnability, generalizing the results of \cite{moran2016sample} to the setting with dependent data (Theorem \ref{thm:pac-learn}). This generalization immediately also implies a generalization of the result for binary hypothesis classes (Corollary \ref{cor:vc-learn}).

A sample compression scheme is a specific type of learner which works by first carefully selecting a small subset of the training samples and then returning a hypothesis which depends only on this subset but performs well on the entire training set. The careful selection is to ensure the existence of a hypothesis which depends only on the selected small subset whole loss is minimized over the whole training set. And if the selected subset is of size $o(m)$, then we can show that any hypothesis chosen based solely on this subset will necessarily have a small generalization error. Together we get learnability.
In the i.i.d. setting, for multiclass hypotheses and the 0/1 loss function, \cite{littlestone1986relating,david2016statistical} show that agnostic learnability is equivalent to the existence of a sublinear size sample compression scheme. We extend this result to the setting of Dobrushin dependent data achieving nearly the same asymptotic rates as in the i.i.d. setting.

To proceed formalizing the discussion above, we begin with the definition of a sample compression scheme for a certain hypothesis class $\cH$  in the general agnostic setting. It consists of two functions: a compressor $\k$ which carefully subsamples the training set and a reconstructor $\r$ which outputs a hypothesis based on this subsample chosen by the compressor. The compressor's job is to select a subsample of the training set such that it allows the reconstructor to output a hypothesis which attains optimal loss on the entire training set. Intuitively, the `simpler' the true underlying function of the data, the compressor should be able to compress the training set to a smaller size. We give a formal definition of sample compression schemes below.

\begin{definition}[Agnostic Sample Compression Scheme]
	\label{def:compression-scheme}
	Fix a hypothesis class $\cH$, integers $0<k<m$ and functions $\kappa \colon (\cX\times \cY)^m \to (\cX \times \cY)^{k}$ and $\rho \colon (\cX \times \cY)^{k} \to \cY^{\cX}$. We say that $(\kappa, \rho)$ is an agnostic sample compression scheme for $\cH$ of size $k$ with respect to a sample-size $m$ if the following hold:
	\begin{itemize}
		\item For all samples $S$, $\kappa(S) \subseteq S$. 
		\item For all samples $S$, $L_S(\r(\k(S))) \le \inf_{h \in \cH} L_S(h)$.
	\end{itemize}
\end{definition}


To understand what hypothesis classes $\cH$ have small sample compression schemes one can look at the instructive setting of binary hypothesis classes, i.e. $\cY = \{0,1\}$. As we point out in Theorems~\ref{thm:compression-vc} and \ref{thm:vc-compression}, having a small VC-dimension is equivalent to having a small compression scheme.
\begin{theorem}[Folklore]
	\label{thm:compression-vc}
	If a class $\cH$ has a sample compression scheme of size $d \in \mathbb{N}$ then $\VC(\cH) = O(d)$.
\end{theorem}

\begin{theorem} [Theorem 1.3 of \cite{moran2016sample}]
	\label{thm:vc-compression}
	Any class of VC dimension $d$ has a sample compression scheme of size $O\left(d2^{(d+1)}\right)$.
\end{theorem}

For many reasonable values of the sample size $m$, Theorem~\ref{thm:vc-compression2} gives a better guarantee than Theorem~\ref{thm:vc-compression}.
\begin{theorem}[\cite{freund1995boosting}]
	\label{thm:vc-compression2}
	Any class of VC dimension $d$ has a sample compression scheme of size $O(d \log m)$.
\end{theorem}

Our main result of this Section is Theorem~\ref{thm:pac-learn} which states that hypothesis classes with sample compression schemes of size $k$ are agnostic PAC-learnable to error $\epsilon$ from $\widetilde{O}(k/\epsilon^2)$ samples.

\begin{theorem}[Agnostic PAC-Learning for Compressible Hypothesis Classes]
	\label{thm:pac-learn}
	Let $\cH$ be a hypothesis class with a sample compression scheme $(\k,\r)$ of size $k$ and let $\ell$ denote the 0/1 loss function. Given a sample $\rv{S}=\{ (\rv{x}_1,\rv{y}_1),\ldots,(\rv{x}_m,\rv{y}_m) \} \sim \cD^{(m)}$ where $\cD^{(m)}$ satisfies Dobrushin's condition with coefficient $\a$, there exists a constant $C(\a)$ such that,
	\begin{align*}
		\Pr\left[ L_D( \r(\k(\rv{S})) ) \ge \inf_{h \in \cH} L_D(h) +  \epsilon \right] \le \d,
	\end{align*}
	\[ \text{for } \; \; \; \; m = \frac{C(\a)k\log (k/\epsilon^2) + \log(1/\d)}{(1-\a)\epsilon^2}. \]
\end{theorem}

Given Theorem~\ref{thm:pac-learn}, Theorems~\ref{thm:vc-compression} and \ref{thm:vc-compression2} immediately give Corollary~\ref{cor:vc-learn}.
\begin{corollary}[Agnostic PAC-Learning for Finite VC-Dimension Classes]
	\label{cor:vc-learn}
	Let $\cH$ be a binary hypothesis class with $VC(\cH) = d$, and let $\ell$ be the 0/1 loss function. Given a sample \\$\rv{S}=\{ (\rv{x}_1,\rv{y}_1),\ldots,(\rv{x}_m,\rv{y}_m) \} \sim \cD^{(m)}$ where $\cD^{(m)}$ satisfies Dobrushin's condition with Dobrushin coefficient $\a$, we have
	\begin{align*}
		\Pr\left[ L_D( \r(\k(\rv{S})) ) \ge \inf_{h \in \cH} L_D(h) +  \epsilon \right] \le \d,
	\end{align*}
	\[ \text{for }\; \; \; \; m = \frac{C(\a)k\log (k/\epsilon^2) + \log(1/\d)}{(1-\a)\epsilon^2} \]
	for some constant $C(\a)$ and for $k = \min\left(d\log m, d2^{(d+1)} \right)$.
\end{corollary}

The proof of Theorem~\ref{thm:pac-learn} proceeds in three steps. Our ultimate goal is to bound the quantity $L_D(\r(\k(\rv{S}))) - \inf_{h \in \cH} L_D(h)$.
\begin{enumerate}
	\item The first step outlines conditions which suffice to show that any compression scheme of size $k = o(m)$ will generalize, i.e. $L_D(\r(\k(\rv{S}))) - L_{\rv{S}}(\r(\k(\rv{S})))$ is small. We believe the identification of sufficient conditions for generalization in Lemma~\ref{lem:compression-generalization-conditions} could lead to the study of learnability under other types of dependencies. 
	\item The second step involves showing that Dobrushin's condition implies the pre-conditions of Lemma~\ref{lem:compression-generalization-conditions} yielding the conclusion that sample compression schemes for Dobrushin distributed data generalize. These two steps are the crucial parts of the proof which differ significantly from the i.i.d. case.
	\item  The third and final step is showing that $L_{\rv{S}}(\r(\k(\rv{S}))) - \inf_{h \in \cH} L_D(h)$ is small. This follows from (a) the definition of a valid compression scheme that it must achieve optimal training error: $L_S(\r(\k(\rv{S})))  \le \inf_{h \in \cH} L_S(h)$; (b) a tail bound on $\inf_{h \in \cH} L_S(h) - \inf_{h \in \cH} L_D(h)$.

\end{enumerate}
To simplify notation we will refer to the Dobrushin coefficient $\a\left(D^{(m)}\right)$ as just $\a$ in the proof.
 
 To show the first step, we first present Lemma~\ref{lem:compression-generalization-conditions}.

\begin{lemma}[Conditions for Generalization of Sample Compression Schemes]
	\label{lem:compression-generalization-conditions}
	Consider a sample $\rv{S}=\{ (\rv{x}_1,\rv{y}_1),\ldots,(\rv{x}_m,\rv{y}_m) \} \sim D^{(m)}$ and any loss function $\ell$ bounded by $R \ge 0$. For any subset of indices $I \subseteq [m]$, let $ \rv{S}_I = \{ (x_i,y_i) \colon i \in I \}$. If we have that for any $I \subseteq [m]$, and for constants $C_1$ and $C_2$,
	\begin{align*}
	& \text{1. } \abss{\EE\left[ L_{\rv{S}}(h) \right] - \EE\left[ L_{\rv{S}}(h) \vert \rv{S}_I \right]  } \le \frac{C_2R|I|}{m},\\
	&\text{2. } \Pr\left[ \abss{L_{\rv{S}}(h) - \EE[L_{\rv{S}}(h) \vert \rv{S}_I] } \ge t \bigm\vert \rv{S}_I\right]  \le 2\exp\left(-\frac{t^2m}{2C_1R^2}\right),
	\end{align*}
	then, for any agnostic sample compression scheme $(\k,\r)$ of size $k$ on $\rv{S}$, we have that, for some constant $C$,
	\begin{align*}
	\Pr\left[ \bigr\vert L_{\rv{S}}( \r(\k(\rv{S})) ) - L_D( \r(\k(\rv{S})) ) \bigl\vert  \ge CR\sqrt{ \frac{\left(k\log m + \log(1/\d)\right)}{m}} \right] \le \d.
	\end{align*}
\end{lemma}
\begin{proof}	
	Given a sample of size $m$, any compression scheme $(\k,\r)$ of size $k$ can select one among ${m \choose k}$ choices of subsets of $\rv{S}$ of size $k$. The total number of different choices $\k$ can make are at most ${m \choose k} \le m^k$.
	For any set of indices $I \subseteq [m]$, where $|I| = k$, define $\hat{h}_I$ as follows: given a sample $\rv{S}$, let $ \rv{S}_I = \{ (x_i,y_i) \colon i \in I \}$. Then, $\hat{h}_I = \rho(\rv{S}_I)$. We will show that for each $I \subset [m]$, $\hat{h}_I$ generalizes. That is, we will show that for any $t > 0$ and any $I$
	\begin{align}
	\Pr\left[ \left| L_{\rv{S}}(\hat{h}_I) - L_D(\hat{h}_I)\right| > \frac{RkC_2 + \sqrt{m}t}{m} \right] \le 2\exp\left(-\frac{t^2}{2C_1R^2}\right). \label{eq:oblivious-gen}
	\end{align}
	Assume without loss of generality that $I = \{ 1, \dots, k\}$. 
	Let $\rv{S}_I = \{ (\rv{x}_1,\rv{y}_1), \dots, (\rv{x}_{k},\rv{y}_{k}) \}$. We look at $m\abss{L_{\rv{S}}(\hat{h}_I) - L_D(\hat{h}_I)}$.
	\begin{align}
		m\left|L_{\rv{S}}(\hat{h}_I) - L_D(\hat{h}_I)\right| &\le \left|\sum_{i=1}^m \left(\ell(\hat{h}_I(\rv{x}_i), \rv{y}_i) - \EE[\ell(\hat{h}_I(\rv{x}_i), \rv{y}_i) \mid \rv{S}_I]\right) \right| \label{eq:decom1}\\
		&+ \left|\sum_{i=1}^m \left(\EE[\ell(\hat{h}_I(\rv{x}_i), \rv{y}_i) \mid \rv{S}_I] - \EE[\ell(\hat{h}_I(\rv{x}_i), \rv{y}_i)]\right) \right|. \label{eq:decom2}
	\end{align}
From property 1 of $D^{(m)}$, we get that $\eqref{eq:decom2} \le C_2Rk$. It is easy to see that $\EE[\eqref{eq:decom1} | \rv{S}_I] = 0$.
From property 2 of $D^{(m)}$, we get that the tail of \eqref{eq:decom1} is bounded as follows:
$$\Pr\left[ \eqref{eq:decom1} \ge t \bigm\vert \rv{S}_I\right]  \le 2\exp\left(-\frac{t^2}{2C_1R^2m}\right).$$
Combining the two we get \eqref{eq:oblivious-gen}.
	Given (\ref{eq:oblivious-gen}), we can union bound over all the possible choices of indices $I$ our compression scheme could make.  The number of such choices, we recall, is upper bounded by $m^k$. Let $\epsilon = RkC_2/m + R\sqrt{2C_1\left(k\log m + \log(1/\d)\right)/m}$. Then,
	\begin{align}
	&\Pr\left[ \abss{L_{\rv{S}}( \r(\k(\rv{S})) ) - L_D( \r(\k(\rv{S})) )} \ge  \epsilon \right] \\
	&\le \Pr\left[\exists \: I \subset [m], |I| = k \colon \abss{ L_{\rv{S}}(\hat{h}_I) - L_D(\hat{h}_I) } \ge  \epsilon \right] \\
	&\le \sum_{I \subset [m], |I| = k} \Pr\left[ \abss{ L_{\rv{S}}(\hat{h}_I) - L_D(\hat{h}_I) } \ge  \epsilon \right] \le m^k \frac{\d}{m^k} \le \d. \label{eq:cig12}
	\end{align}
	Since $k = o(m)$ and $(1-\a)$ is a constant bounded away from 0, the dominant term in the value of $\epsilon$ is the second one as $m$ grows. Hence we can re-write \eqref{eq:cig12} as
	\begin{align}
	\Pr\left[ \abss{L_{\rv{S}}( \r(\k(\rv{S})) ) - L_D( \r(\k(\rv{S})) )} \ge \frac{C(\a)R\sqrt{2m\left(k\log m + \log(1/\d)\right)}}{m}  \right] \le \d,
	\end{align}
	for a large enough constant $C$.
\end{proof}

Next, we show that if the data distribution is Dobrushin, the conditions of  f Lemma~\ref{lem:compression-generalization-conditions} are satisfied.
We will use two properties of Dobrushin distributions to show this. The first, stated as Lemma~\ref{lem:dobrushin-implies-cond-dobrushin} states that all conditional distributions of a Dobrushin distribution also satisfy Dobrushin's condition with the same coefficient.
\begin{restatable}{lemma}{doblemmaone}[Conditionining Preserves Low Influence Property]
	\label{lem:dobrushin-implies-cond-dobrushin}
	Consider a distribution $\pi$ defined on $\Omega^n$ which satisfies Dobrushin's condition with coefficient $\alpha$. Let $\rv{x} \sim \pi$ and let $0 < k \le n$. Let $(a_1,a_2,\ldots,a_k) \in \Omega^k$ be such that $\Pr_{\pi}[(x_1,x_2,\ldots,x_k)=(a_1,a_2,\ldots,a_k)] > 0$.
	Then the conditional probability distribution $\pi_{\vec{a}} \eqdef \Pr_{\pi}[. | (x_1,x_2,\ldots,x_k)=(a_1,a_2,\ldots,a_k)]$ also satisfies Dobrushin's condition with coefficient $\alpha$.
\end{restatable}
We give the proof in Section \ref{sec:app-learnability}.
The next property is more technically involved. It considers the Hamming distance between two samples drawn from a conditional distribution under two distinct conditionings. It shows that this quantity can be bounded to be of the order of the size of the set of conditioned variables.
\begin{lemma}[Bounding Expected Hamming Distance Between two Conditional Measures]
	\label{lem:hamm-exp-bound}
	Let $a,a' \in (\cX \times \cY)^k$ and let $\rv{U} \sim D^{(m)}\left( . |(z_i)_{i=1}^k = (a_i)_{i=1}^k  \right)$ and $\rv{V}\sim D^{(m)}\left( . |(z_i)_{i=1}^k = (a_i')_{i=1}^k  \right)$. Then, there exists a coupling $(\rv{Z_1},\rv{Z_2})$ such that,
	$$\EE[d_H(\rv{U},\rv{V})] \le \frac{k\a}{1-\a},$$
	where $d_H(x,y) = \sum_{i=k+1}^m \mathbbm{1}_{x_i \ne y_i}$ is the Hamming distance between $x$ and $y$.
\end{lemma}
 The full proof of Lemma~\ref{lem:hamm-exp-bound} is given in Section \ref{sec:app-learnability}. We give an outline here. The key technical ingredent is Markov chain coupling. We start by observing that one way to generate a sample from the conditional distribution is to start at an arbitrary configuration and run a Gibbs sampler until mixing. Given the two conditionings, we consider two such Markov chains starting from the same configuration and hence the configurations have a Hamming distance of 0 in the beginning. The chains run in a coupled manner so that the Hamming distance between the configurations can be shown to remain small in expectation after each step of updates. Once we run until both chains have mixed the random configurations at that point correspond to samples from the conditional distributions but at every step along the way the Hamming distance between them has been bounded and hence it will remain bounded after the chains have mixed as well giving the result. To achieve this, it was necessary to find the appropriate coupling. We use what is known as the greedy coupling and is popularly used to show fast mixing of Gibbs sampling for Dobrushin distributions \cite{LevinPW09}. This coupling tries to maximize the probability of agreement of the two chains updates' at each step. 
 
 With Lemmas \ref{lem:dobrushin-implies-cond-dobrushin} and \ref{lem:hamm-exp-bound}, we are ready to show Lemma \ref{lem:dobrushin-properties}.
 
\begin{lemma}
	\label{lem:dobrushin-properties}
	Let $D^{(m)}$ be a distribution over $m$ variables which satisfies Dobrushin's condition with coefficient $\a$ such that the marginal of every variable is $D$. Let $\rv{S} \sim D^{(m)}$. Then we have
	\begin{align*}
	& \text{1. } \abss{\EE\left[ L_{\rv{S}}(h) \right] - \EE\left[ L_{\rv{S}}(h) \vert \rv{S}_I \right]  } \le \frac{(2-\a) R|I|}{(1-\a)m},\\
	&\text{2. } \Pr\left[ \abss{L_{\rv{S}}(h) - \EE[L_{\rv{S}}(h) \vert \rv{S}_I]} \ge t \bigm\vert \rv{S}_I\right]  \le 2\exp\left(-\frac{t^2m(1-\a)}{2R^2}\right),
	\end{align*}
\end{lemma}
\begin{proof}
	Let $|I|= k$. Again, we assume without loss of generality that $I = \{ 1, \dots, k\}$. The proof idea remains the same for any other set $I$.
	We have
	\begin{align}
		m\abss{\EE\left[ L_{\rv{S}}(h) \right] - \EE\left[ L_{\rv{S}}(h) \vert \rv{S}_I \right]  } \le& \left|\sum_{i=1}^{k} \left(\EE[\ell(\hat{h}_I(\rv{x}_i), \rv{y}_i)] - \EE[\ell(\hat{h}_I(\rv{x}_i), \rv{y}_i) \vert \rv{S}_I]\right)\right| \label{eq:decom3} \\
		&+ \left|\sum_{i=k+1}^m \left(\EE[ \ell(\hat{h}_I(\rv{x}_i), \rv{y}_i)] - \EE[\ell(\hat{h}_I(\rv{x}_i), \rv{y}_i) \vert \rv{S}_I]\right) \right|. \label{eq:decom4}
	\end{align}
	Since $0 \le \ell(.) \le R$, we get that $\eqref{eq:decom3} \le 2Rk$.
	\eqref{eq:decom4} is bounded using Lemma~\ref{lem:hamm-exp-bound} as follows. Let $\rv{S}_I$ and $\tilde{\rv{S}}_I$ be two realizations of the sample on the set of indices $I$. Then,
		\begin{align}
		&\abss{\sum_{i=k+1}^m \left(\EE[\ell(\hat{h}_I(\rv{x}_i), \rv{y}_i) \mid \rv{S}_I] - \EE[\ell(\hat{h}_I(\rv{x}_i), \rv{y}_i)]\right)} \\
		&~~\le \sup_{\tilde{\rv{S}}_I \ne \rv{S}_I} \sum_{i=k+1}^m \left(\EE[\ell(\hat{h}_I(\rv{x}_i), \rv{y}_i) \mid \rv{S}_I] - \EE[\ell(\hat{h}_I(\rv{x}_i), \rv{y}_i) \mid \tilde{\rv{S}}_I]\right) \label{eq:cig8}\\
		&~~\le \sup_{\tilde{\rv{S}}_I \ne \rv{S}_I} \EE\left[d_H(\rv{z}_{i=k+1}^m, \rv{\tilde{z}}_{i=k+1}^m) L \mid \rv{S}_I, \tilde{\rv{S}}_I\right]  \le \frac{Rk\a}{1-\a}, \label{eq:cig9}
		\end{align}
		where \eqref{eq:cig8} utilized the fact that for two random variables $\rv{x}$ and $\rv{y}$ where $\rv{x} \ge 0$, $\abss{\EE[\rv{x}] - \EE[\rv{x} \mid \rv{y}=y_1 ]} \le \sup_{y_2}  \abss{\EE[\rv{x} \mid \rv{y}=y_2] - \EE[\rv{x} \mid \rv{y}=y_1 ]}$ because $\EE[\rv{x}]$ is a convex combination of values of the form $\EE[\rv{x} | \rv{y}=y]$. In \eqref{eq:cig9}, $\rv{z}_{i=k+1}^m$ is sampled from the distribution conditioned on $\rv{S}_I$ and $\rv{\tilde{z}}_{i=k+1}^m$ is sampled from the distribution conditioned on $\tilde{\rv{S}}_I$. The first inequality in \eqref{eq:cig9} follows by coupling the two conditional probability spaces together with employing the fact that $|\ell| \le R$, and the second is proved in Lemma~\ref{lem:hamm-exp-bound}.
		Putting the two bounds together, we get the first conclusion of the Lemma.
		\begin{align}
			 \abss{\EE\left[ L_{\rv{S}}(h) \right] - \EE\left[ L_{\rv{S}}(h) \vert \rv{S}_I \right]  } \le \frac{Rk(2-\a)}{m(1-\a)}.
		\end{align}
		
	Next, we show the second conclusion of the Lemma which involves the distribution obtained by conditioning on $\rv{S}_I$. Recall that $I = \{ 1,2,\ldots,k \}$.
	Firstly, we notice that due to the conditioning
	\begin{align}
		m\left(L_{\rv{S}}(h) - \EE[L_{\rv{S}}(h) \vert \rv{S}_I] \right) = \sum_{i=k+1}^m \left(  \ell(\hat{h}_I(\rv{x}_i,\rv{y}_i)) - \EE[\ell(\hat{h}_I(\rv{x}_i), \rv{y}_i) \vert \rv{S}_I] \right). \label{eq:decom5}
	\end{align}
	Denote $\mu(\rv{S}_I) = \sum_{i=k+1}^m\EE[\ell(\hat{h}_I(\rv{x}_i), \rv{y}_i) \vert \rv{S}_I]$. Note that conditioned on $\rv{S}_I$, $\hat{h}_I$ is fixed. 
	Next we invoke Lemma~\ref{lem:dobrushin-implies-cond-dobrushin} to argue that since $((\rv{x}_i,\rv{y}_i))_{i=1}^m$ comes from a distribution satisfying Dobrushin's condition with coefficient $\a$, the conditional distribution of $((\rv{x}_i,\rv{y}_i))_{i=k+1}^m | \rv{S}_I$ satisfies Dobrushin's condition with coefficient $\a$ as well.  
	Hence we get the following concentration bound for \eqref{eq:decom5} by employing Theorem~\ref{thm:dob-conc} for Dobrushin distributions.
	\[
	\Pr\left[ \left|\sum_{i=k+1}^m \ell(\hat{h}_I(\rv{x}_i,\rv{y}_i)) - \mu(\rv{S}_I)\right| > \left(\sqrt{m-k}\right) t ~\middle|~ \rv{S}_I \right]
	\le 2\exp\left(-\frac{t^2 (1-\a)}{2R^2} \right) .
	\]
	Replacing $\sqrt{m-k}$ with $\sqrt{m}$, one obtains the second conclusion of the Lemma.
\end{proof}

Lemmas \ref{lem:compression-generalization-conditions} and \ref{lem:dobrushin-properties} imply Corollary \ref{cor:compression-implies-generalization}.
\begin{corollary}[Sample Compression Schemes Generalize under Dobrushin]
	\label{cor:compression-implies-generalization}
	Consider a sample $\rv{S}=\{ (\rv{x}_1,\rv{y}_1),\ldots,(\rv{x}_m,\rv{y}_m) \} \sim D^{(m)}$ which satisfies Dobrushin's condition with coefficient $\a\left(D^{(m)}\right)$. 
	For any agnostic sample compression scheme $(\k,\r)$ of size $k$ on a sample $\rv{S}$, and any loss function $\ell$ bounded by $R \ge 0$, we have that, for some constant $C(\a)$ depending on $\a$,
	\begin{align*}
	\Pr\left[ \abss{L_{\rv{S}}( \r(\k(\rv{S})) ) - L_D( \r(\k(\rv{S})) )} \ge C(\a)R\sqrt{ \frac{\left(k\log m + \log(1/\d)\right)}{m}} \right] \le \d.
	\end{align*}
\end{corollary}

Lemmas \ref{lem:compression-generalization-conditions} and \ref{lem:dobrushin-properties} together imply Corollary \ref{cor:compression-implies-generalization}.
Corollary \ref{cor:compression-implies-generalization} together with the property of an agnostic sample compression scheme implies Theorem~\ref{thm:pac-learn}.
\begin{proof}
	\textbf{of Theorem~\ref{thm:pac-learn}}. Let $\epsilon = \epsilon_1 + \epsilon_2$ where 
	\begin{align*}
		&\epsilon_1 = C(\a)R\sqrt{\left(k\log m + \log(2/\d)\right)/m}, \\
		&\epsilon_2 = 2R\sqrt{(1-\a)^{-1} \log(2/\d)/m}.
	\end{align*}
	Since $(\k,\r)$ is an agnostic sample compression scheme, we have that $L_{\rv{S}}(\r(\k(\rv{S}))  ) \le \inf_{h \in \cH} L_{\rv{S}}(h)$. We will show that this implies that $L_{\rv{S}}(\r(\k(\rv{S}))  ) - \inf_{h \in \cH} L_{D}(h) \le \epsilon_2$ with probability $\ge 1-\d/2$. Let $h^* = \argmin_{h \in \cH} L_D(h)$.
	\begin{align}
	L_{\rv{S}}(\r(\k(\rv{S}))  ) \le \inf_{h \in \cH} L_{\rv{S}}(h) \le  L_{\rv{S}}(h^*). \notag
	\end{align}
	For any $h \in \cH$, we have using Theorem~\ref{thm:dob-conc} that
	\begin{align}
	\Pr\left[ \abss{	L_{\rv{S}}(h) - L_D(h)}  \ge \epsilon_2 \right] \le \d/2. \label{eq:main3}
	\end{align}
	Hence,
	\begin{align}
		&\Pr\left[ L_D( \r(\k(\rv{S})) ) \ge \inf_{h \in \cH} L_D(h) +   \epsilon \right] \notag\\
		=& \Pr\left[  L_D( \r(\k(\rv{S})) ) \ge L_{\rv{S} }( \r(\k(\rv{S})) ) +  \epsilon + \left(\inf_{h \in \cH} L_D(h) - L_{\rv{S} }( \r(\k(\rv{S}))  \right) \right] \notag\\
		\le& \Pr\left[  L_D( \r(\k(\rv{S})) ) \ge L_{\rv{S} }( \r(\k(\rv{S})) ) +  \epsilon_1 \right] +  \Pr\left[ L_{\rv{S} }( \r(\k(\rv{S}))  \ge  \inf_{h \in \cH} L_D(h) + \epsilon_2 \right] \notag\\
		\le &  \d/2 + \Pr\left[ L_{\rv{S}}(h^*) \ge L_D(h^*) + \epsilon_2\right] \le \d. \label{eq:main4}
	\end{align}
   where the first inequality in \eqref{eq:main4} follows from Corollary~\ref{cor:compression-implies-generalization} and the second inequality follows from \eqref{eq:main3}.

\end{proof}

\subsection{Extending Learnability to Bounded Loss Functions}
\label{sec:general-loss}
Learnability under Dobrushin's condition can also be extended for general loss functions which are bounded by some constant $L$. This can be shown using $\epsilon$-approximate sample compression schemes defined in Section 4 of \cite{david2016statistical}. Lemma~\ref{lem:compression-generalization-conditions}  continues to hold for these sample compression schemes and the rest of the result follows suit.
\section{Uniform Convergence for Weakly Dependent Data}
\label{sec:uc}

In this section, we obtain uniform convergence bounds for the empirical loss on weakly dependent training data. We could not derive such bounds for distributions satisfying Dobruhsin's condition and we do not know if such bounds apply for all classes of finite VC dimension. Hence, we present such bounds for a smaller family of distributions, which contain high temperature Markov Random Fields with pairwise potentials, however, allows an arbitrary structure of correlations (as long as they are sufficiently weak). We define the log-influences $I_{j, i}^{\log}$, a notion stronger than Dobrushin's influences $I_{j\to i}$, which replaces the total variation distance appearing in the definition with a stronger bound on the maximal log-ratio of probabilities. Analogously, we obtain the log-coefficient $\alpha_{\log}$, as defined below:
\begin{definition}[Log-influence and log-coefficient]
	Let $\rv{z} = (\rv{z}_1,\dots,\rv{z}_m)$ be a random variable over $\Omega^m$ and let $P_{\rv{z}}$ denote either its probability distribution if discrete or its density if continuous. Assume that $P_z > 0$ on all $\Omega^m$.
	For any $i \ne j \in [m]$, define the \emph{log-influence} between $j$ and $i$ as\footnote{To be more formal, one can define $P_{\rv{z}} = d\mu$ where $\rv{z} \sim \mu$ and replace the supremum with an essential supremum.}
	\[
	I^{\log}_{j, i}(\rv{z})
	= \frac{1}{4} \sup_{\substack{z_{-i-j} \in \Omega^{m-2} \\ z_i,z_i', z_j,z_j' \in \Omega}}
	\log \frac{P_{\rv{z}}[z_i z_j z_{-i-j}]P_{\rv{z}}[z_i' z_j' z_{-i-j}]}{P_{\rv{z}}[z_i' z_j z_{-i-j}]P_{\rv{z}}[z_i z_j' z_{-i-j}]}.
	\]
	Define the \emph{log-coefficient} of $\rv{z}$ as $\alpha_{\log}(\rv{z}) = \max_{i\in [m]} \sum_{j \ne i} I^{\log}_{j,i}(\rv{z})$.
\end{definition}

Note that the log influence is symmetric: $I_{j,i}^{\log} = I_{i,j}^{\log}$. The following relation holds:
\begin{lemma} \label{lem:inf-bnd}
	For any random variable $\rv{z}$ and $i,j \in [m]$, $I_{j \to i}(\rv{z}) \le I^{\log}_{j, i}(\rv{z}) \le \beta_{j,i}(\rv{z})$.
\end{lemma}
The proof is simple and appears in Section~\ref{sec:pr-influence-compare}. 
The main result of this section shows that uniform convergence holds whenever the log-coefficient is less than a half. In that regime, the maximal generalization error of hypotheses from $H$, $\sup_{h \in H} |L_S(h) - L_D(h)|$, is bounded in terms of the Gaussian complexity of $H$:
\begin{theorem} \label{thm:main}
	Let $\cH$ be a hypothesis class, let $\ell \colon \cY^2 \to [-L,L]$ be a loss function, and let  $\mathcal{L}_H = \{\ell_h \colon h \in H\}$. Let $D^{(m)}$ be a distribution over $(X \times Y)^m$, with all $m$ marginals equaling $D$ and $\alpha_{\log}(D^{(m)}) < 1/2$. Then, for all $t > 0$,
	\[
	\Pr_{\rv{S} \sim D^{(m)}} \left[ \sup_{h \in H} \left|L_S(h) - L_D(h)\right|
	> C \left( \GC_{D^{(m)}}(\mathcal{L}_H)
	+ \frac{L t}{\sqrt{m} } \right)
	\right]
	\le e^{-t^2 / 2},
	\]
	(where $C$ is a universal constant whenever $1/2-\alpha_{\log}\left(D^{(m)}\right)$ is bounded away from zero).
\end{theorem}
The proof appears in Section~\ref{sec:pr-thm-main-cor} and is a direct corollary of Theorem~\ref{thm:pairwise} which is presented below.
Note that Lemma~\ref{lem:inf-bnd} implies that Theorem~\ref{thm:main} also holds whenever $D^{(m)}$ is an MRF with pairwise potentials and $\beta(D^{(m)}) < 1/2$. Since the condition $\beta(D^{(m)}) < 1$ sufficient for concentration inequalities to hold, we suspect that Theorem~\ref{thm:main} may hold as well in this regime. However, when $\beta(D^{(m)}) > 1$, Theorem~\ref{thm:main} is not generally true, since concentration inequalities are not guaranteed to hold.

Applying Theorem~\ref{thm:main} on any hypothesis class $H$ with finite VC, one obtains the same sample complexity bounds of i.i.d data up to constant factors: 
\begin{equation} \label{eq:sid}
O\left(\frac{\mathrm{VC}(H) + \log (1/\delta)}{\varepsilon^2}\right).
\end{equation}
This follows from the fact that the Gaussian complexity of $\mathcal{L}_H$ is bounded by $O\left(\sqrt{VC(H)/m}\right)$. The proof is almost identical to the proof bounding the Rademacher complexity by the same quantity (see, for instance, \citealp{shalev2014understanding}, Chapter 27).

Although the Rademacher and Gaussian complexities are not identical, they are almost equivalent.
Both notions were introduced to the learning community by \cite{bartlett2002rademacher}, and \cite{tomczak1989banach} proved the following:
\[
c\eRad_S(\cF) \le \eGC_S(\cF) \le C \ln m ~\eRad_S(\cF),
\]
for some universal constants $c, C > 0$. \cite{bednorz2014boundedness} resolved the long-standing Bernoulli conjecture by Talagrand and gave an exact characterization of the relation between these two notions.
The standard techniques for bounding the Rademacher complexity, based on Chernoff-Hoeffding bounds, apply for bounding the Gaussian complexity as well, with the same constants (this includes chaining and covering numbers).

Theorem~\ref{thm:main} is based on a more general result, bounding the expected suprema of empirical processes with respect to the corresponding Gaussian complexity. Here, the supremum is taken over an arbitrary family of unbounded functions, rather than bounded loss functions. Also, the $m$ marginals of the weakly correlated distribution are not assumed to be identical. Hence, one can derive a variant of Theorem~\ref{thm:main} with non-identical marginals.

\begin{theorem} \label{thm:pairwise}
	Let $D^{(m)}$ be a random vector over some domain $\cZ^m$ and let $\cF$ be a class of functions from $\cZ$ to $\mathbb{R}$. If $\alpha_{\log}(D^{(m)}) < 1/2$, then
	\begin{equation}\label{eq:main}
	\EE_{\rv{S} \sim D^{(m)}} \sup_{f \in \cF} \left(\frac{1}{m} \sum_{i=1}^m f(\rv{s_i}) - \EE_{\rv{S}} \left[\frac{1}{m} \sum_{i=1}^m f(\rv{s_i})\right]\right)
	\le \frac{C \GC_{D^{(m)}}(\cF)}{\sqrt{1 - 2\alpha_{\log}(D^{(m)})}},
	\end{equation}
	where $C > 0$ is a universal constant.
\end{theorem}
The proof appears in Section~\ref{sec:pr-uc}. Theorem~\ref{thm:main} follows from Theorem~\ref{thm:pairwise} simply by applying a McDiarmind-like inequality on weakly correlated data satisfying Dobrushin's condition (Theorem~\ref{thm:dob-conc}).


\section{Proofs for Section~\ref{sec:uc}} \label{sec:pr-uc}

In Section~\ref{sec:prel-stochastic-proc} we present some preliminaries for the proof. In Section~\ref{pr:thm-pairwise} we prove the main result, Theorem~\ref{thm:pairwise}. Then, in Section~\ref{sec:pr-influence-compare} we present the proof of Lemma~\ref{lem:inf-bnd} and in Section~\ref{sec:pr-thm-main-cor} the proof of Theorem~\ref{thm:main}.

\subsection{Preliminaries: sub Gaussian distributions and stochastic processes} \label{sec:prel-stochastic-proc}

A joint distribution $P^{(m)}$ over $\mathbb{R}^m$ is a $K^2$-subGaussian if has subGaussian tails in any direction:
\begin{definition}
	A zero-mean distribution $P^{(m)}$ over $\mathbb{R}^m$ is a $K^2$-subGaussian if for any $\theta \in \mathbb{R}^n$ ($\theta \ne 0$) and any $t > 0$,
	\[
	\Pr_{\rv{w} \sim P^{(m)}}\left[ \left|\sum_{i=1}^m \theta_i \rv{w}_i\right| > t \right]
	\le 2 \exp \left( \frac{-t^2}{2 K^2 \sum_{i=1}^m \theta_i^2}\right).
	\]
\end{definition}

A \emph{stochastic process} is a collection of joint random variables, $\{\rv{w}_i \}_{i \in I}$, taking values in $\mathbb{R}$, with some (possibly infinite) index-set $I$. A basic quantity of interest when talking about stochastic processes is the supremum, and in particular, the expected supremum, $\EE \sup_{i \in I} \rv{w}_i$. We will be focusing on \emph{zero-mean} processes, namely, those which satisfy $\EE \rv{w}_i = 0$ for all $i \in I$. We present two important types of stochastic processes: Gaussian and subGaussian processes. A \emph{Gaussian} process is a stochastic process where the variables are jointly Gaussian, namely, for any finite $U \subseteq I$, the collection $\{\rv{w}_i \}_{i \in U}$ is a multivariate Gaussian variable. A \emph{subGaussian process} is a stochastic process for which $w_i - w_j$ is subGaussian for all $i,j \in I$. The following statement by \citet{talagrand1996majorizing} upper bounds the expected maximum of a subGaussian process by that of a corresponding Gaussian process:

\begin{theorem}[The majorizing measure theorem] \label{thm:majorizing-measure}
	Fix $I$ to be some index set and let $\{\rv{w}_i\}_{i \in I}$ and $\{\rv{g}_i\}_{i \in I}$ be subGaussian and Gaussian zero-mean processes, respectively. For any $i,j \in I$, let $\sigma_{ij}^2$ denote the variance of $\rv{g}_i - \rv{g}_j$. Assume that for any $i,j \in I$, $\rv{w}_i - \rv{w}_j$ is a $\sigma_{ij}^2$-subGaussian random variable. Then,
	\[
	\EE\left[ \sup_{i \in I} \rv{w}_i \right]
	\le C \EE\left[ \sup_{i \in I} \rv{g}_i \right],
	\]
	for a universal constant $C>0$.
\end{theorem}

\subsection{Proof of Theorem~\ref{thm:pairwise}} \label{pr:thm-pairwise}

Here is the proof structure: first, we bound the left hand side of \eqref{eq:main} by the $\rv{\sigma}$-complexity of $\mathcal{F}$ (Eq.~\eqref{eq:p-rad}), where $\rv{\sigma}$ does not consist of i.i.d random signs, but rather it is a subGaussian distribution with zero mean (Lemma~\ref{lem:subg-log} and the explanation afterwards). Furthermore, this $\rv{\sigma}$-complexity is not with respect to $D^{(m)}$ but rather with respect to a different distribution. Then, we bound this $\rv{\sigma}$-complexity by the Gaussian complexity of $\mathcal{F}$, with respect to $D^{(m)}$ (Lemma~\ref{lem:subg-implies-rad} and Lemma~\ref{lem:bnd-complexities}).

Assume that $\rv{S} = (\rv{s}_i)_{i \in [m]} \sim D^{(m)}$, and $\rv{S'} = (\rv{s'}_i)_{i \in [m]} $ is another i.i.d. random variable drawn from $D^{(m)}$.
The following holds:
\begin{multline} \label{eq:sup-S-Sp}
\EE_{\rv{S}} \sup_{f \in \cF} \left(
\frac{1}{m}\sum_{i=1}^m f(\rv{s}_i) - \EE_{\rv{S}} \frac{1}{m}\sum_{i=1}^m f(\rv{s}_i) \right)
= \EE_{\rv{S}} \sup_{f \in \cF} \left(
\frac{1}{m}\sum_{i=1}^m f(\rv{s}_i) - \EE_{\rv{S}'} \frac{1}{m}\sum_{i=1}^m f(\rv{s}'_i) \right) \\
\le \EE_{\rv{S}, \rv{S'}} \sup_{f \in \cF} \left(\frac{1}{m}\sum_{i=1}^m f(\rv{s}_i) - \frac{1}{m}\sum_{i=1}^m f(\rv{s}'_i)\right).
\end{multline}
Indeed, one can verify that the supremum of a sum is at most the sum of supremums.
We randomly shuffle $\rv{S}$ and $\rv{S}'$, to create samples $\rv{T}$ and $\rv{T}'$. 
Formally, $m$ i.i.d and uniform random signs are drawn, $\rv{\sigma} = (\rv{\sigma}_1, \dots, \rv{\sigma}_m) \in \{-1,1\}^m$. Then, $\rv{T} = (\rv{t}_1,\dots, \rv{t}_m)$ and $\rv{T}' = (\rv{t}'_1,\dots, \rv{t}'_m)$ are defined as functions of $\rv{S}$, $\rv{S}'$ and $\rv{\sigma}$, as follows: for any $i \in \{1,\dots, m\}$, if $\rv{\sigma}_i = 1$ then $\rv{t}_i = \rv{s}_i$ and $\rv{t}'_i = \rv{s}'_i$, and otherwise, $\rv{t}_i = \rv{s}'_i$ and $\rv{t}'_i = \rv{s}_i$.

For any $T$ and $T'$ denote by $\rv{\sigma}_{T,T'}$ a random variable sampled from $P_{\rv{\sigma} \mid \rv{T} \rv{T}'}(\cdot \mid T,T')$, the conditional distribution of $\rv{\sigma}$, conditioned on $\rv{T} = T$ and $\rv{T}' = T'$.
We bound the right hand side of \eqref{eq:sup-S-Sp}, substituting $\rv{S}$ and $\rv{S}'$ with $\rv{T}$ and $\rv{T}'$, in a \emph{change of measure} argument:
\begin{multline} \label{eq:p-rad}
\EE_{\rv{S},\rv{S}'} \left[ \sup_{f \in \cF} \left(\frac{1}{m}\sum_{i=1}^m f(\rv{s}_i) - \frac{1}{m}\sum_{i=1}^m f(\rv{s}'_i) \right)\right] 
= \EE_{\rv{T}, \rv{T'}, \rv{\sigma}}\left[ \sup_{f \in \cF} \frac{1}{m}\sum_{i=1}^m \rv \sigma_i \left(f(\rv{t}_i) - f(\rv{t}'_i)\right) \right] \\
\le \EE_{\rv{T}, \rv{T}', \rv{\sigma}}\left[ \sup_{f \in \cF} \frac{1}{m}\sum_{i=1}^m \rv \sigma_i f(\rv{t}_i) \right]
+ \EE_{\rv{T}, \rv{T'}, \rv{\sigma}}\left[ \sup_{f \in \cF} \frac{1}{m}\sum_{i=1}^m \rv \sigma_i (-f(\rv{t}'_i)) \right] \\
=_{(*)} 2 \EE_{\rv{T}, \rv{T'}, \rv{\sigma}}\left[ \sup_{f \in \cF} \frac{1}{m}\sum_{i=1}^m \rv \sigma_i f(\rv{t}_i) \right]
= 2 \EE_{\rv{T}, \rv{T}'} \eCom_{\rv{T}}^{\rv{\sigma}_{\rv{T},\rv{T'}}}(\cF),
\end{multline}
where the equality $~(*)~$ follows from the fact that the joint distribution of $\rv T$ and $\rv\sigma$ equals the joint distribution of $\rv T'$ and $-\rv\sigma$. 
Note that $\rv{\sigma}_{T,T'}$ is generally not a product distribution, however, we can show that it is a zero mean subGaussian.

\begin{lemma} \label{lem:subg-log}
	For any $T$ and $T'$, $\rv{\sigma}_{T,T'}$ is zero-mean and satisfies $I_{j, i}^{\log}(\rv{\sigma}_{T,T'}) \le 2 I_{j, i}^{\log}(D^{(m)})$ for any $i \ne j$.
\end{lemma}

\begin{proof}
	Fix $T$ and $T'$. By definition of $\rv{\sigma}$, for any $\sigma \in \{-1,1\}^m$, 
	$\Pr[\rv{T} = T, \rv{T}' = T', \rv{\sigma} = \sigma] = 
	\Pr[\rv{T} = T, \rv{T}' = T', \rv{\sigma} = -\sigma]$. This implies that $\Pr[\rv{\sigma}_{T,T'} = \sigma] = \Pr[\rv{\sigma}_{T,T'} = -\sigma]$, which implies that $\rv{\sigma}_{T,T'}$ is zero-mean.
	
	Next, we prove the inequality on the influence. For any $\sigma \in \{-1,1\}^m$, define $S(\sigma,T,T')$ and $S'(\sigma,T,T')$ as the values that $\rv{S}$ and $\rv{S}'$ get when $\rv{T} = T$, $\rv{T}' = T'$ and $\rv{\sigma} = \sigma$.
	
	Fix $i \ne j$ and fix $\sigma_i, \sigma_i', \sigma_j, \sigma_j' \in \{-1,1\}$ and $\sigma_{-i-j} \in \{-1,1\}^{m-2}$. Then,
	\begin{align*}
	&\frac{1}{4}\log \left(\frac{
		P_{\rv{\sigma}_{T,T'}}(\sigma_{-i-j} \sigma_i \sigma_j)
		P_{\rv{\sigma}_{T,T'}}(\sigma_{-i-j} \sigma_i' \sigma_j')}{
		P_{\rv{\sigma}_{T,T'}}(\sigma_{-i-j} \sigma_i' \sigma_j)
		P_{\rv{\sigma}_{T,T'}}(\sigma_{-i-j} \sigma_i \sigma_j')
	}\right) \\
	&=\frac{1}{4}\log \left(\frac{
		P_{\rv{\sigma}, \rv{T}, \rv{T}'}(\sigma_{-i-j} \sigma_i \sigma_j, T, T')
		P_{\rv{\sigma}, \rv{T}, \rv{T}'}(\sigma_{-i-j} \sigma_i' \sigma_j', T, T')}{
		P_{\rv{\sigma}, \rv{T}, \rv{T}'}(\sigma_{-i-j} \sigma_i' \sigma_j, T, T')
		P_{\rv{\sigma}, \rv{T}, \rv{T}'}(\sigma_{-i-j} \sigma_i \sigma_j', T, T')
	}\right) \\
	&= \frac{1}{4}\log \left(\frac{
		P_{\rv{S}}(S(\sigma_{-i-j} \sigma_i \sigma_j, T, T'))
		P_{\rv{S}}(S(\sigma_{-i-j} \sigma_i' \sigma_j', T, T'))}{
		P_{\rv{S}}(S(\sigma_{-i-j} \sigma_i' \sigma_j, T, T'))
		P_{\rv{S}}(S(\sigma_{-i-j} \sigma_i \sigma_j', T, T'))
	}\right) \\
	&+\frac{1}{4}\log \left(\frac{
		P_{\rv{S}'}(S'(\sigma_{-i-j} \sigma_i \sigma_j, T, T'))
		P_{\rv{S}'}(S'(\sigma_{-i-j} \sigma_i' \sigma_j', T, T'))}{
		P_{\rv{S}'}(S'(\sigma_{-i-j} \sigma_i' \sigma_j, T, T'))
		P_{\rv{S}'}(S'(\sigma_{-i-j} \sigma_i \sigma_j', T, T'))
	}\right) \\
	&\le 2 I_{j, i}^{\log}(D^{(m)}),
	\end{align*}
	where the last step follows from the definition of the log-influences. The proof follows.
\end{proof}

If follows from Lemma\ref{lem:inf-bnd} and Lemma~\ref{lem:subg-log} that $\alpha(\rv{\sigma}_{T,T'}) \le \alpha^{\log}(\rv{\sigma}_{T,T'}) \le 2 \alpha^{\log}(D^{(m)})$. From Theorem~\ref{thm:dob-conc} it follows that $\rv{\sigma}_{T,T'}$ is a $C/(1-\alpha(\rv{\sigma}_{T,T'}))$-subGaussian, hence it is a $C/(1-2\alpha^{\log}(D^{(m)}))$-subGaussian.
We will use this to show that the $\rv{\sigma}_{T,T'}$ complexity of $\mathcal{F}$ can be bounded in terms of the Gaussian complexity. This will bound the right hand side of \eqref{eq:p-rad}. The proof follows from the Fernique-Talagrand Majorizing measure theory.

\begin{lemma} \label{lem:subg-implies-rad}
	Fix $z \in \cZ^m$. If $\rv{\tau}$ is a $K^2$-subgaussian, then,
	\[
	\eCom^{\rv{\tau}}_z(\cF)
	\le CK \eGC_{z}(\cF),
	\]
	(for some universal constant $C > 0$).
	In particular, $\eCom^{\rv{\sigma}_{T,T'}}_z(\cF)
	\le C \eGC_{z}(\cF) / \sqrt{1-2\alpha^{\log}(D^{(m)})}$.
\end{lemma}

\begin{proof}
	Define the random process, $\{ \rv{w}_f \}_{f \in \cF}$, where each $\rv{w}_f$ equals $\frac{1}{m}\sum_{i=1}^m f(z_i) \rv{\tau}_i$. Let $\rv{g} \sim N(0, \Id_m)$, and define the Gaussian process $\{\rv{g}_f \}_{f \in \cF}$ by $\rv{g}_f = \frac{1}{m}\sum_{i=1}^m \rv{g}_i f(z_i)$. 
	Note that by the definition of a subGaussian random variable, for any $f, f' \in \cF$ and any $t > 0$,
	\[
	\Pr\left[ \left|\rv{w}_f - \rv{w}_{f'}\right| > t \right]
	= \Pr\left[ \left|\frac{1}{m}\sum_{i=1}^m \rv{\tau}_i \left( f(z_i) - f'(z_i) \right)\right| > t \right]
	\le 2 \exp\left( \frac{m^2 t^2}{2 K^2 \sum_{i=1}^m \left( f(z_i) - f'(z_i) \right)^2} \right),
	\]
	which implies that $\rv{w}_f - \rv{w}_{f'}$ is a $K^2 \sum_{i=1}^m \left( f(z_i) - f'(z_i) \right)^2/m^2$-subGaussian.
	Additionally, 
	\[
	\mathrm{Var}(\rv{g}_f - \rv{g}_{f'})
	= \mathrm{Var}\left(\frac{1}{m}\sum_{i=1}^m \rv{g}_i \left( f(z_i) - f'(z_i) \right)\right)
	= \frac{1}{m^2}\sum_{i=1}^m \left( f(z_i) - f'(z_i) \right)^2.
	\]
	Theorem~\ref{thm:majorizing-measure} implies that
	\[
	\eCom^{\rv{\tau}}_z(\cF)
	= \EE \sup_{f \in \cF} \rv{w}_f \le C \EE\sup_{f \in \cF} K\rv{g}_f
	= CK \eGC_z(\cF).
	\]
\end{proof}


Lemma~\ref{lem:subg-implies-rad} implies that the right hand side of \eqref{eq:p-rad} is bounded as follows:
\begin{equation}\label{eq:15}
\EE_{\rv{T}, \rv{T}'} \eCom_{\rv{T}}^{\rv{\sigma}_{\rv{T},\rv{T'}}}(\cF)
\le \frac{C}{\sqrt{1 - 2\beta(D^{(m)})}}\EE_{\rv{T}, \rv{T}'} \eGC_{\rv{T}}(\cF) = 
\frac{C}{\sqrt{1 - 2\beta(D^{(m)})}}\GC_{\rv{T}}(\cF).
\end{equation}
We will bound this last term by the Gaussian complexity of $D^{(m)}$.
\begin{lemma} \label{lem:bnd-complexities}
The following holds:
\[
\GC_{\rv{T}}(\cF) \le 2 \GC_{\cD^{(m)}}(\cF).
\]	
\end{lemma}
\begin{proof}
	Recall that $\rv{T} = T(\rv{S},\rv{S}',\rv{\sigma})$ is a mixture of two samples $\rv{S}$ and $\rv{S}'$ drawn from $\cD^{(m)}$: $\rv{t}_i = \rv{s}_i$ if $\rv{\sigma}_i = 1$ and otherwise $\rv{t}_i = \rv{s}'_i$. 
	Fix $S, S' \in Z^m$ and $\sigma \in \{-1,1\}^m$, and let $T = T(S,S',\sigma)$. Taking expectation over $\rv{g \sim \normal(0, \Id_m)}$, one obtains:
	\begin{multline*}
		\EE_{\rv{g} \sim \normal(0,\Id_m)} \left[ \sup_{f \in \cF} \sum_{i=1}^m \rv{g}_i f(t_i) \right]
		= \EE_{\rv{g} \sim \normal(0,\Id_m)} \left[ \sup_{f \in \cF} \left( \sum_{i \colon \rv{\sigma}_i = 1} \rv{g}_i f(s_i) + \sum_{i \colon \rv{\sigma}_i = -1} \rv{g}_i f(s'_i) \right)\right] \\
		= \EE_{\rv{g} \sim \normal(0,\Id_m)} \left[ \sup_{f \in \cF} \left( \sum_{i \colon \rv{\sigma}_i = 1} \rv{g}_i f(s_i) 
		+ \sum_{i \colon \rv{\sigma}_i = -1} \rv{g}_i f(s'_i) 
		+ \EE_{\rv{g}' \sim \normal(0,\Id_m)} \left[
		\sum_{i \colon \rv{\sigma}_i = -1} \rv{g}'_i f(s_i) 
		+ \sum_{i \colon \rv{\sigma}_i = 1} \rv{g}'_i f(s'_i)
		\right]
		\right)\right] \\
		\le \EE_{\rv{g}, \rv{g}' \in \normal(0,\Id_m)} \left[ \sup_{f \in \cF} \left( \sum_{i \colon \rv{\sigma}_i = 1} \rv{g}_i f(s_i) 
		+ \sum_{i \colon \rv{\sigma}_i = -1} \rv{g}_i f(s'_i) 
		+ \sum_{i \colon \rv{\sigma}_i = -1} \rv{g}'_i f(s_i) 
		+ \sum_{i \colon \rv{\sigma}_i = 1} \rv{g}'_i f(s'_i)
		\right)\right] \\
		= \EE_{\rv{g}, \rv{g}' \sim \normal(0,\Id_m)} \left[ \sup_{f \in \cF} \left( \sum_{i = 1}^m \rv{g}_i f(s_i) 
		+ \sum_{i =1}^m \rv{g}'_i f(s'_i) 
		\right)\right] 
		\le 2 \EE_{\rv{g} \sim \normal(0,\Id_m)} \left[ \sup_{f \in \cF} \left( \sum_{i = 1}^m \rv{g}_i f(s_i) 
		\right)\right].
	\end{multline*}
	Taking expectation in both sides over $\rv{T} = T$ and $\rv{S} = S$, the result follows.
\end{proof}

The proof concludes by equations \eqref{eq:sup-S-Sp}, \eqref{eq:p-rad}, \eqref{eq:15} and Lemma~\ref{lem:bnd-complexities}.

\subsection{Proof of Lemma~\ref{lem:inf-bnd}} \label{sec:pr-influence-compare}

First, we bound the log-influences $I_{j, i}^{\log}(\rv{z})$ by $\beta_{j,i}(\rv{z})$ for an MRF with pairwise potentials. Assume $\rv{z}$ has pairwise potentials $\psi_{ij}(z_i,z_j)$ and node-wise potentials $\varphi_i(z_i)$. Fix $i \ne j$, $z_i,z_i', z_j, z_j'$ and $z_{-i-j}$, and note that
\[
	\frac{1}{4} \log \frac{P_{\rv{z}}[z_i z_j z_{-i-j}]P_{\rv{z}}[z_i' z_j' z_{-i-j}]}{P_{\rv{z}}[z_i' z_j z_{-i-j}]P_{\rv{z}}[z_i z_j' z_{-i-j}]}
	= \frac{1}{4} \left(\psi_{ij}(z_i z_j) + \psi_{ij}(z_i' z_j') - \psi_{ij}(z_i' z_j) - \psi_{ij}(z_i z_j')\right) \le \beta_{ij}(\rv{z}).
\]
This concludes that $I^{\log}_{j,i}(\rv{z}) \le \beta_{ij}(\rv{z})$. Next, we bound the Dobrushin influences with respect to the log-influences. We begin with the following auxiliary lemma.
\begin{lemma} \label{lem:8}
	Let $\{M_{a, b}\}_{a,b \in \{-1,1\}}$ be positive numbers. Then,
	\begin{equation} \label{eq:37}
	\left|\frac{M_{1,1}}{M_{1,1}+M_{-1,1}} - \frac{M_{1,-1}}{M_{1,-1}+M_{-1,-1}}\right|
	\le \frac{1}{4} \max \left\{ \log \frac{M_{1,1} M_{-1,-1}}{M_{1,-1} M_{-1,1}}, \log \frac{M_{1,-1} M_{-1,1}}{M_{1,1} M_{-1,-1}} \right\}.
	\end{equation}
\end{lemma}

\begin{proof}
	We can assume that $\sum_{a,b} M_{a,b} = 1$, by scaling. Define a random variable $\rv{w}$ over $\{-1,1\}^2$ with $\Pr[\rv{w} = (a,b)] = M_{a,b}$. Any random variable with two binary coordinates can be written as an Ising model. In particular, there exists $\theta \in \mathbb{R}$, and $\varphi_1,\varphi_2 \colon \{-1,1\} \to \mathbb{R}_+$ such that $\Pr[\rv{w} = (a,b)] = e^{\varphi_1(a)+ \varphi_2(b) + ab \theta}$. 
	Note that the left hand side of \eqref{eq:37} equals $I_{2\to 1}(\rv{w})$, and from Lemma~\ref{lem:hightemp-implies-dobrushin}, it is bounded by $|\theta|$. On the other hand, the right hand size of \eqref{eq:37} equals $|\theta|$, and the proof follows.
\end{proof}

Fix $z_{-i-j} \in \Omega^{m-2}$ and $z_j,z_j' \in \Omega$ and let $F$ be the event such that
\[
d_{TV}(P_{\rv{z}_i \mid \rv{z}_{-i}}(\cdot \mid z_{-i-j} z_j), P_{\rv{z}_i \mid \rv{z}_{-i}}(\cdot \mid z_{-i-j} z_j'))
= P_{\rv{z}_i \mid \rv{z}_{-i}}(F \mid z_{-i-j} z_j) - P_{\rv{z}_i \mid \rv{z}_{-i}}(F \mid z_{-i-j} z_j').
\]
The proof follows by applying Lemma~\ref{lem:8} with
\begin{align*}
	M_{1,1} = \Pr[\rv{z}_i \in F, \rv{z}_j = z_j, \rv{z}_{-i-j} = z_{-i-j}] ; \quad&
	M_{1,-1} = \Pr[\rv{z}_i \in F, \rv{z}_j = z_j', \rv{z}_{-i-j} = z_{-i-j}] \\
	M_{-1,1} = \Pr[\rv{z}_i \in F^\mathsf{c}, \rv{z}_j = z_j, \rv{z}_{-i-j} = z_{-i-j}] ; \quad&
	M_{-1,-1} = \Pr[\rv{z}_i \in F^\mathsf{c}, \rv{z}_j = z_j', \rv{z}_{-i-j} = z_{-i-j}].
\end{align*}

\subsection{Proof of Theorem~\ref{thm:main}} \label{sec:pr-thm-main-cor}

We prove a slightly more general result.

\begin{theorem} \label{thm:high-probabiliy}
	Assume the same setting as in Theorem~\ref{thm:pairwise} and additionally, that there exists $L > 0$ such that for every $f \in \cF$ and $z \in Z$, $|f(z)| \le L$. Then, for any $t > 0$,
	\[
	\Pr_{\rv{S} \sim D^{(m)}} \left[ \sup_{f \in \cF} \left|\sum_{i=1}^m f(\rv{s_i}) - \EE_{\rv{S}} \left[\sum_{i=1}^m f(\rv{s_i})\right]\right|
	> \frac{C \GC_{D^{(m)}}(\cF)}{\sqrt{1 - 2\alpha_{\log}(D^{(m)})}} + CL \sqrt{m} t
	\right]
	\le e^{-t^2 / 2},
	\]
	for some universal constant $C > 0$.
\end{theorem}

\begin{proof}
	We start by bounding the probability that $\sup_{f \in \cF} \left(\sum_{i=1}^m f(\rv{s_i}) - \EE_{\rv{S}} \left[\sum_{i=1}^m f(\rv{s_i})\right]\right)$ is larger than the corresponding value, removing the absolute value (later, we will argue for the opposite inequality).
	The proof follows from a McDiarmid-like inequality for dependent distributions. Define the function $M \colon Z^m \to \mathbb{R}$ by 
	\[ 
	M(S) 
	= \sup_{f \in \cF} \left(\sum_{i=1}^m f(s_i) - \EE_{\rv{S}} \sum_{i=1}^m f(s_i)\right).
	\] 
	For any $S = (s_1, \dots, s_m)$ and $S' = (s_1',\dots,s_m') \in Z^m$, it holds that $\left| M(S) - M(S')\right| \le \sum_{i=1}^m 2L \mathbbm{1}_{s_i \ne s_i'}$. Lemma~\ref{lem:inf-bnd} and Theorem~\ref{thm:dob-conc} imply that for any $t' > 0$,
	\[
	\Pr_{\rv{S} \sim D^{(m)}}\left[M(\rv{S}) - \EE M(\rv{S}) > t' \right]
	\le \exp\left( \frac{-t'^2\left( 1 - \alpha_{\log}\left( D^{(m)} \right) \right)}{C' L^2 m} \right)
	\le \exp\left( \frac{-t'^2 }{2C' L^2 m} \right).
	\]
	Substituting $\EE M(S)$ using Theorem~\ref{thm:high-probabiliy} and setting $t := t' \sqrt{1 / (C' L^2 m)}$, the bound concludes.
	
	To bound the opposite inequality, namely, the probability that
	$\sup_{f \in \cF} \left(\EE_{\rv{S}} \left[\sum_{i=1}^m f(\rv{s_i})\right] - \sum_{i=1}^m f(\rv{s_i})\right)$ is large, one can apply the same arguments on $- \cF := \{ -f \colon f \in \cF \}$, and note that the Gaussian complexity of $\cF$ equals that of $-\cF$.
\end{proof}
\section{Comparison to Related Work on Time Series}
\label{sec:timeseries-comparison}

Much of the work on non-iid Rademacher complexity has focused on time series.
A {\em time series} is a distribution $D^{(\infty)}$ on random variables that are indexed by integer times.
In our setting, we let the random sequence be $\rv{z} = \ldots, \rv{z_{-2}}, \rv{z_{-1}}, \rv{z_0}, \rv{z_1}, \rv{z_2}, \ldots$, where each $\rv{z_i} = (\rv{x_i}, \rv{y_i})$.
Such a series is said to be {\em stationary} if for any times $t$ and $t'$ and any positive integer $k$, the joint distribution of $(\rv{z_t},\ldots,\rv{z_{t+k}})$ and $(\rv{z_t'},\ldots,\rv{z_{t'+k}})$ are the same.
A time series is said to be {\em mixing} if events that are farther spread out in time are closer to being independent from one another.
More formally, for any $t, t' \in \mathbb{Z} \cup \{-\infty, +\infty \}$, let $\sigma_t^{t'}$ be the $\sigma$-algebra generated by $\{\rv{z_i} \vert t \le i \le t' \}$.
The $\alpha$-function is defined as $\alpha(k) = \sup_t \{\abss{\Pr(A,B) - \Pr(A)\Pr(B)} \mid B \in \sigma_{-\infty}^t, A \in \sigma_{t+k}^\infty\}$, and
the $\beta$-function is defined as $\beta(k) = \sup_t \EE_{B \in \sigma_{-\infty}^t} \sup_{A \in \sigma_{t+k}^\infty} \abss{\Pr(A | B) - \Pr(A)}$.
A time series $D^{(\infty)}$ is said to be $\alpha$-mixing if $\lim_{k \rightarrow \infty} \alpha(k) = 0$ and $\beta$-mixing if $\lim_{k \rightarrow \infty} \beta(k) = 0$.
It is a well known that for all $k$ $\beta(k) > \alpha(k)$, and thus $\beta$-mixing implies $\alpha$-mixing.  

Since our work focuses on distributions with identical marginals, we compare it to previous work on stationary time series.
Of these works, \cite{mohri2009rademacher} is most relevant to ours, since McDonald et. al.'s work on non-mixing time series solves the forecasting problem, predicting $\rv{z}_{m+1}$  given previous data, rather than on predicting $\rv{y}$ given $\rv{x}$ as in our setting.	
\cite{mohri2009rademacher} derive the following uniform convergence Rademacher complexity bound for stationary $\beta$-mixing time series.


\begin{theorem}[Theorem 1 in \cite{mohri2009rademacher}]
	Let $H$ be a hypothesis class, $\ell$ a loss function bounded by $L \ge 0$,
	and $\cF = \{\ell \circ h \mid h \in H \}$.
	Then, for any size $m$ sample $S$ from a stationary $\beta$-mixing time series $D^{(\infty)}$ with marginal distribution $D$,
	and any $\mu, a > 0$ with $2\mu a = m$ and $\delta > 2(\mu - 1)\beta(a)$,
	then with probability at least $1 - \delta$, the following inequality holds for all $h \in H$:
	
	$$ \abss{L_D(h) - L_S(h)} \le \Rad_{D^\mu}(\cF) + L \sqrt{\frac{\log \frac{2}{\delta - 2(\mu - 1)\beta(a)}}{2\mu}}$$
	\label{thm:Mohri-UC}
\end{theorem}

Mohri {\em et. al.} derive their result by exploiting the fact that samples that are sufficiently far apart in the time series are close to independent.
They split the sample $(\rv{z_1},\ldots,\rv{z_m})$ into large blocks, and argue that if one data point is taken from each block, then the resultant subsample is close to being iid.
Once they ``thin'' the original sample in this way, they finish the argument by appealing to methods for proving Rademacher bounds for iid distributions.

While Mohri et. al.'s thinning argument makes analysis simple and yields Rademacher bounds, our work significanlty improves the sample complexity bounds if the time series is also an MRF with pairwise potentials, or, if it is Dobrushin - at least for learnability.
We demonstrate a super-quadratic improvement in the following example.

\subsection{Example: Uniform Convergence Sample Complexity Comparison}
\label{subsec:time-series-example}

Given a symmetric real matrix $\Theta \in \mathbb{R}^{2n+1 \times 2n+1}$,
we define the distribution $P_{\Theta,n}$ on the $2n + 1$ variables $((\rv{x_{-n}}, \rv{y_{-n}}),\ldots,(\rv{x_{n}}, \rv{y_n}))$ as follows. 
Each $\rv{x_i}$ takes a value in the set $[-1,1]$ with the probability density function 
$p(x_{-n},\ldots,x_n) \propto \prod_{i \ne j} e^{\theta_{i,j} x_i x_j}$.
Then, we let $\rv{y_i} = f(\rv{x_i})$, where $f: [-1,1] \rightarrow \{-1, +1\}$ is any function.

In order for Theorem~\ref{thm:main} to apply to $P_{\Theta, n}$, we must ensure that the sums of the pairwise potentials including any one node $i$ are less than (and bounded away from) $0.5$.
So, we let $\theta_{i,j} =\displaystyle\frac{c}{|i - j| \log^2 (|i - j| + 1)}$ for each $i \ne j$, where $c$ is some constant that ensures the convergent series 
$\displaystyle\sum_{k = 1}^\infty \frac{c}{k\log^2 (k + 1)}$ is bounded by a sufficiently small constant.

While the $\beta(k)$ coefficients are bounded by $o(1)$ for $P_{\Theta, n}$, the distribution is not techincally $\beta$-mixing, since a $\beta$-mixing distribution must technically be a distribution on a countable collection of variables (not just on $2n + 1$ variables).
This technicality could be resolved by taking the limiting distribution that arises when $n$ tends to infinity.
The resultant distribution would also be stationary by the symmetry of the $\theta$s.
In order to avoid technicalities of probability theory in favor of more clearly illustrating the sample complexity differences, we just consider a huge (but finite) value of $n$, and call that resulting distribution $D^{(n)}$.

Let $H \subseteq \{h: [-1,1] \rightarrow \{-1,+1\} \}$ be any hypothesis class of finite VC-dimension $d$, and consider the 0-1 loss.
A sample complexity bound $m(\eps, \delta)$ is a bound on the number of samples needed to ensure that the generalization gap is less than $\eps$ with probability at least $1 - \delta$.
We compare our sample complexity bound obtained from Equation~\eqref{eq:sid}, with Mohri et. al.'s bound, derived from Theorem~\ref{thm:Mohri-UC}.
A key to the comparison is the following fact. 

\begin{claim}
	\label{lem:example-is-slow-mixing}
	For the distribution $D$, $\beta(k) \ge \alpha(k) = \Omega(\theta_{0, k}) = \Omega\left(\frac{c}{k \log^2 (k+1)}\right)$.
\end{claim}

\begin{proof}
	{\bf Sketch}
	It is always true that the $\beta$ mixing coefficient is larger than the $\alpha$ mixing coefficient.
	In order to show that the $\alpha$-mixing coefficient is not too small, we study the events 
	$E_i \equiv (\rv{x_i} > 0)$.
	In particular, we note that $E_0 \in \sigma_{-\infty}^0$ and $E_k \in \sigma_{k}^\infty$ and we show that $\Pr(E_0, E_k) - \Pr(E_0) \Pr(E_k) = \Omega(\theta_{0,k})$.
	The density function is symmetric, i.e. $p((\rv{x_i})_{i=-n}^n = (a_i)_{i=-n}^n) = p((\rv{x_i})_{i=-n}^n = (-a_i)_{i=-n}^n)$,
	so $\Pr(E_0) \Pr(E_k) = 1/4$.
	
	It now suffices to show that $\Pr(E_0, E_k) = 1/4 + \Omega(\theta_{0,k})$.
	We first observe that since all the $\theta_{i,j}$ coefficients are non-negative in the distribution $P_{\Theta, n}$,
	$\Pr(E_0, E_k)$ can only be made smaller if all but the coefficient $\theta_{0,k}$ are made zero.
	Let this new distribution be called $P$.
	Under $P$, the probability 
	\begin{align}
		\Pr(E_0, E_k) &= \frac{\int_0^1\int_0^1 e^{\theta_{0,k} x y} dx dy}{\int_{-1}^1\int_{-1}^1 e^{\theta_{0,k} x y} dx dy}	\\
		&= \frac{1 + \theta_{0,k}/4 + O(\theta_{0,k}^2)}{4 + O(\theta_{0, k}^2)} \text{ (Taylor expansion)}				\\
		&= 1/4 + \theta_{0, k}/16 + O(\theta_{0, k}^2)																						\\
		&= 1/4 + \Omega(\theta_{0, k})
	\end{align}

\end{proof}

We ignore polylogarithmic-factors for clarity in the following calculations.
Since, $\cH$ has VC-dimension $d$, Mohri et. al.'s Theorem~\ref{thm:Mohri-UC} implies that the generalization gap satisfies 
$\eps \le \sqrt{\frac{d}{\mu}} + \sqrt{\frac{1}{\mu}}$.
By the conditions of the theorem,
the block size $a$ and the number of blocks $\mu$ must satisfy $\mu a = m$ and $\mu \beta(a) \le \delta$.
Adding in the constraint $\beta(a) = \tilde{\Omega}(1/a)$ of Lemma~\ref{lem:example-is-slow-mixing},
yields their tightest sample complexity bound

$$m_{\text{prior-work}}(\eps, \delta) = \tilde{\Theta}\left(\frac{d^2}{\delta\eps^4}\right).$$

Alternately, our sample complexity bound from Equation~\eqref{eq:sid} is 

$$m_{\text{this-paper}}(\eps, \delta) = \Theta\left(\frac{d + \log \frac{1}{\delta}}{\eps^2}\right).$$

The set of stationary $\beta$-mixing time series studied by Mohri et. al. is not a subset of the Dobrushin and pairwise-potential-MRF distributions we study in this paper.
However, in this example, where the time series is also a pairwise-potential-MRF our analysis improves the sample requirement quadratically in $\eps$ and $d$, and {\em exponentially} in $1/\delta$.
The quadratic improvement quantifies the ineffciency of the thinning method in this context,
while the exponential improvement in the dependence on $1/\delta$ results from our use of powerful measure concentration inequalities in our analysis.

\bibliographystyle{plain}
\bibliography{biblio,ising}

\newcommand{\noopsort}[1]{} \newcommand{\printfirst}[2]{#1}
  \newcommand{\singleletter}[1]{#1} \newcommand{\switchargs}[2]{#2#1}
\begin{thebibliography}{44}
\providecommand{\natexlab}[1]{#1}
\providecommand{\url}[1]{\texttt{#1}}
\expandafter\ifx\csname urlstyle\endcsname\relax
  \providecommand{\doi}[1]{doi: #1}\else
  \providecommand{\doi}{doi: \begingroup \urlstyle{rm}\Url}\fi

\bibitem[Agarwal and Duchi(2013)]{agarwal2013generalization}
Alekh Agarwal and John~C Duchi.
\newblock The generalization ability of online algorithms for dependent data.
\newblock \emph{IEEE Transactions on Information Theory}, 59\penalty0
  (1):\penalty0 573--587, 2013.

\bibitem[Bartlett and Mendelson(2002)]{bartlett2002rademacher}
Peter~L Bartlett and Shahar Mendelson.
\newblock Rademacher and gaussian complexities: Risk bounds and structural
  results.
\newblock \emph{Journal of Machine Learning Research}, 3\penalty0
  (Nov):\penalty0 463--482, 2002.

\bibitem[Bednorz and Latala(2014)]{bednorz2014boundedness}
Witold Bednorz and Rafal Latala.
\newblock On the boundedness of bernoulli processes.
\newblock \emph{Annals of Mathematics}, 180\penalty0 (3):\penalty0 1167--1203,
  2014.

\bibitem[Berti et~al.(2009)Berti, Crimaldi, Pratelli, Rigo,
  et~al.]{berti2009rate}
Patrizia Berti, Irene Crimaldi, Luca Pratelli, Pietro Rigo, et~al.
\newblock Rate of convergence of predictive distributions for dependent data.
\newblock \emph{Bernoulli}, 15\penalty0 (4):\penalty0 1351--1367, 2009.

\bibitem[Bertrand et~al.(2000)Bertrand, Luttmer, and
  Mullainathan]{bertrand2000network}
Marianne Bertrand, Erzo~FP Luttmer, and Sendhil Mullainathan.
\newblock Network effects and welfare cultures.
\newblock \emph{The Quarterly Journal of Economics}, 115\penalty0 (3):\penalty0
  1019--1055, 2000.

\bibitem[Bramoull{\'e} et~al.(2009)Bramoull{\'e}, Djebbari, and
  Fortin]{bramoulle2009identification}
Yann Bramoull{\'e}, Habiba Djebbari, and Bernard Fortin.
\newblock Identification of peer effects through social networks.
\newblock \emph{Journal of econometrics}, 150\penalty0 (1):\penalty0 41--55,
  2009.

\bibitem[Chatterjee(2005{\natexlab{a}})]{Chatterjee05}
Sourav Chatterjee.
\newblock \emph{Concentration Inequalities with Exchangeable Pairs}.
\newblock PhD thesis, Stanford University, June 2005{\natexlab{a}}.

\bibitem[Chatterjee(2005{\natexlab{b}})]{chatterjee2005concentration}
Sourav Chatterjee.
\newblock Concentration inequalities with exchangeable pairs (ph. d. thesis).
\newblock \emph{arXiv preprint math/0507526}, 2005{\natexlab{b}}.

\bibitem[Christakis and Fowler(2013)]{christakis2013social}
Nicholas~A Christakis and James~H Fowler.
\newblock Social contagion theory: examining dynamic social networks and human
  behavior.
\newblock \emph{Statistics in medicine}, 32\penalty0 (4):\penalty0 556--577,
  2013.

\bibitem[Daskalakis et~al.(2018)Daskalakis, Dikkala, and
  Kamath]{DaskalakisDK18}
Constantinos Daskalakis, Nishanth Dikkala, and Gautam Kamath.
\newblock Testing {I}sing models.
\newblock In \emph{Proceedings of the 29th Annual ACM-SIAM Symposium on
  Discrete Algorithms}, SODA '18, Philadelphia, PA, USA, 2018. SIAM.

\bibitem[David et~al.(2016)David, Moran, and Yehudayoff]{david2016statistical}
Ofir David, Shay Moran, and Amir Yehudayoff.
\newblock On statistical learning via the lens of compression.
\newblock In \emph{Proceedings of the 30th International Conference on Neural
  Information Processing Systems}, pages 2792--2800. Curran Associates Inc.,
  2016.

\bibitem[Dobrushin(1968)]{dobruschin1968description}
PL~Dobrushin.
\newblock The description of a random field by means of conditional
  probabilities and conditions of its regularity.
\newblock \emph{Theory of Probability \& Its Applications}, 13\penalty0
  (2):\penalty0 197--224, 1968.

\bibitem[Dobrushin and Shlosman(1987)]{dobrushin1987completely}
RL~Dobrushin and SB~Shlosman.
\newblock Completely analytical interactions: constructive description.
\newblock \emph{Journal of Statistical Physics}, 46\penalty0 (5-6):\penalty0
  983--1014, 1987.

\bibitem[Duflo and Saez(2003)]{duflo2003role}
Esther Duflo and Emmanuel Saez.
\newblock The role of information and social interactions in retirement plan
  decisions: Evidence from a randomized experiment.
\newblock \emph{The Quarterly journal of economics}, 118\penalty0 (3):\penalty0
  815--842, 2003.

\bibitem[Freund(1995)]{freund1995boosting}
Yoav Freund.
\newblock Boosting a weak learning algorithm by majority.
\newblock \emph{Information and computation}, 121\penalty0 (2):\penalty0
  256--285, 1995.

\bibitem[Gheissari et~al.(2017)Gheissari, Lubetzky, and Peres]{GheissariLP17}
Reza Gheissari, Eyal Lubetzky, and Yuval Peres.
\newblock Concentration inequalities for polynomials of contracting {I}sing
  models.
\newblock \emph{arXiv preprint arXiv:1706.00121}, 2017.

\bibitem[Glaeser et~al.(1996)Glaeser, Sacerdote, and
  Scheinkman]{glaeser1996crime}
Edward~L Glaeser, Bruce Sacerdote, and Jose~A Scheinkman.
\newblock Crime and social interactions.
\newblock \emph{The Quarterly Journal of Economics}, 111\penalty0 (2):\penalty0
  507--548, 1996.

\bibitem[Harris et~al.(2009)Harris, of~Adolescent~Health,
  et~al.]{harris2009waves}
Kathleen~Mullan Harris, National Longitudinal~Study of~Adolescent~Health,
  et~al.
\newblock Waves i \& ii, 1994--1996; wave iii, 2001--2002; wave iv, 2007--2009
  [machine-readable data file and documentation].
\newblock \emph{Chapel Hill, NC: Carolina Population Center, University of
  North Carolina at Chapel Hill}, 10, 2009.

\bibitem[K{\"u}lske(2003)]{kulske2003concentration}
Christof K{\"u}lske.
\newblock Concentration inequalities for functions of gibbs fields with
  application to diffraction and random gibbs measures.
\newblock \emph{Communications in mathematical physics}, 239\penalty0
  (1-2):\penalty0 29--51, 2003.

\bibitem[K{\"u}nsch(1982)]{kunsch1982decay}
H~K{\"u}nsch.
\newblock Decay of correlations under dobrushin's uniqueness condition and its
  applications.
\newblock \emph{Communications in Mathematical Physics}, 84\penalty0
  (2):\penalty0 207--222, 1982.

\bibitem[Kuznetsov and Mohri(2014)]{kuznetsov2014}
Vitaly Kuznetsov and Mehryar Mohri.
\newblock Generalization bounds for time series prediction with non-stationary
  processes.
\newblock In Peter Auer, Alexander Clark, Thomas Zeugmann, and Sandra Zilles,
  editors, \emph{Algorithmic Learning Theory}, pages 260--274, Cham, 2014.
  Springer International Publishing.
\newblock ISBN 978-3-319-11662-4.

\bibitem[Kuznetsov and Mohri(2015)]{kuznetsov2015learning}
Vitaly Kuznetsov and Mehryar Mohri.
\newblock Learning theory and algorithms for forecasting non-stationary time
  series.
\newblock In \emph{Advances in neural information processing systems}, pages
  541--549, 2015.

\bibitem[Kuznetsov and Mohri(2017)]{Kuznetsov2017}
Vitaly Kuznetsov and Mehryar Mohri.
\newblock Generalization bounds for non-stationary mixing processes.
\newblock \emph{Machine Learning}, 106\penalty0 (1):\penalty0 93--117, Jan
  2017.
\newblock \doi{10.1007/s10994-016-5588-2}.
\newblock URL \url{https://doi.org/10.1007/s10994-016-5588-2}.

\bibitem[Levin et~al.(2009)Levin, Peres, and Wilmer]{LevinPW09}
David~A. Levin, Yuval Peres, and Elizabeth~L. Wilmer.
\newblock \emph{{M}arkov Chains and Mixing Times}.
\newblock American Mathematical Society, 2009.

\bibitem[Li et~al.(2016)Li, Levina, and Zhu]{li2016prediction}
Tianxi Li, Elizaveta Levina, and Ji~Zhu.
\newblock Prediction models for network-linked data.
\newblock \emph{arXiv preprint arXiv:1602.01192}, 2016.

\bibitem[Littlestone and Warmuth(1986)]{littlestone1986relating}
Nick Littlestone and Manfred Warmuth.
\newblock Relating data compression and learnability.
\newblock 1986.

\bibitem[Manski(1993)]{manski1993identification}
Charles~F Manski.
\newblock Identification of endogenous social effects: The reflection problem.
\newblock \emph{The review of economic studies}, 60\penalty0 (3):\penalty0
  531--542, 1993.

\bibitem[Marton et~al.(1996)]{marton1996bounding}
Katalin Marton et~al.
\newblock Bounding $\bar{d}$-distance by informational divergence: A method to
  prove measure concentration.
\newblock \emph{The Annals of Probability}, 24\penalty0 (2):\penalty0 857--866,
  1996.

\bibitem[McDonald and Shalizi(2017)]{McDonald2017}
Daniel~J. McDonald and Cosma~Rohilla Shalizi.
\newblock Rademacher complexity of stationary sequences.
\newblock \emph{arXiv preprint arXiv:1106.0730}, 2017.

\bibitem[Mohri and Rostamizadeh(2009)]{mohri2009rademacher}
Mehryar Mohri and Afshin Rostamizadeh.
\newblock Rademacher complexity bounds for non-iid processes.
\newblock In \emph{Advances in Neural Information Processing Systems}, pages
  1097--1104, 2009.

\bibitem[Mohri and Rostamizadeh(2010)]{mohri2010stability}
Mehryar Mohri and Afshin Rostamizadeh.
\newblock Stability bounds for stationary $\varphi$-mixing and $\beta$-mixing
  processes.
\newblock \emph{Journal of Machine Learning Research}, 11\penalty0
  (Feb):\penalty0 789--814, 2010.

\bibitem[Montanari and Saberi(2010)]{montanari2010}
Andrea Montanari and Amin Saberi.
\newblock The spread of innovations in social networks.
\newblock \emph{Proceedings of the National Academy of Sciences}, 107\penalty0
  (47):\penalty0 20196--20201, 2010.
\newblock ISSN 0027-8424.
\newblock \doi{10.1073/pnas.1004098107}.
\newblock URL \url{https://www.pnas.org/content/107/47/20196}.

\bibitem[Moran and Yehudayoff(2016)]{moran2016sample}
Shay Moran and Amir Yehudayoff.
\newblock Sample compression schemes for vc classes.
\newblock \emph{Journal of the ACM (JACM)}, 63\penalty0 (3):\penalty0 21, 2016.

\bibitem[Pestov(2010)]{pestov2010predictive}
Vladimir Pestov.
\newblock Predictive pac learnability: A paradigm for learning from
  exchangeable input data.
\newblock In \emph{Granular Computing (GrC), 2010 IEEE International Conference
  on}, pages 387--391. IEEE, 2010.

\bibitem[Rakhlin et~al.(2010)Rakhlin, Sridharan, and Tewari]{rakhlin2010online}
Alexander Rakhlin, Karthik Sridharan, and Ambuj Tewari.
\newblock Online learning: Random averages, combinatorial parameters, and
  learnability.
\newblock In \emph{Advances in Neural Information Processing Systems}, pages
  1984--1992, 2010.

\bibitem[Sacerdote(2001)]{sacerdote2001peer}
Bruce Sacerdote.
\newblock Peer effects with random assignment: Results for dartmouth roommates.
\newblock \emph{The Quarterly journal of economics}, 116\penalty0 (2):\penalty0
  681--704, 2001.

\bibitem[Shalev-Shwartz and Ben-David(2014)]{shalev2014understanding}
Shai Shalev-Shwartz and Shai Ben-David.
\newblock \emph{Understanding machine learning: From theory to algorithms}.
\newblock Cambridge university press, 2014.

\bibitem[Stroock and Zegarlinski(1992)]{stroock1992logarithmic}
Daniel~W Stroock and Boguslaw Zegarlinski.
\newblock The logarithmic sobolev inequality for discrete spin systems on a
  lattice.
\newblock \emph{Communications in Mathematical Physics}, 149\penalty0
  (1):\penalty0 175--193, 1992.

\bibitem[Talagrand et~al.(1996)]{talagrand1996majorizing}
Michel Talagrand et~al.
\newblock Majorizing measures: the generic chaining.
\newblock \emph{The Annals of Probability}, 24\penalty0 (3):\penalty0
  1049--1103, 1996.

\bibitem[Tomczak-Jaegermann(1989)]{tomczak1989banach}
Nicole Tomczak-Jaegermann.
\newblock \emph{Banach-Mazur distances and finite-dimensional operator ideals},
  volume~38.
\newblock Longman Sc \& Tech, 1989.

\bibitem[Trogdon et~al.(2008)Trogdon, Nonnemaker, and Pais]{trogdon2008peer}
Justin~G Trogdon, James Nonnemaker, and Joanne Pais.
\newblock Peer effects in adolescent overweight.
\newblock \emph{Journal of health economics}, 27\penalty0 (5):\penalty0
  1388--1399, 2008.

\bibitem[Vapnik and Chervonenkis(2015)]{vapnik2015uniform}
Vladimir~N Vapnik and A~Ya Chervonenkis.
\newblock On the uniform convergence of relative frequencies of events to their
  probabilities.
\newblock In \emph{Measures of complexity}, pages 11--30. Springer, 2015.

\bibitem[Weitz(2005)]{weitz2005combinatorial}
Dror Weitz.
\newblock Combinatorial criteria for uniqueness of gibbs measures.
\newblock \emph{Random Structures \& Algorithms}, 27\penalty0 (4):\penalty0
  445--475, 2005.

\bibitem[Yu(1994)]{yu1994rates}
Bin Yu.
\newblock Rates of convergence for empirical processes of stationary mixing
  sequences.
\newblock \emph{The Annals of Probability}, pages 94--116, 1994.

\end{thebibliography}
\appendix

\section{The Gibbs Sampling Algorithm}
\label{sec:gibbs}

In Section~\ref{sec:learnability} we use the Gibbs sampling algorithm associated with Dobrushin distributions.
For completeness we present the algorithm below.

\begin{algorithm}[H]
\label{alg:gibbs}
	\KwIn{Set of variables $V$, Configuration $x_0 \in S^{|V|}$, Distribution $\pi$}
	initialization\;
	\For{$t = 1$ \KwTo $T$ }{
		Sample $i$ uniformly from $\{1,2,\ldots,n\}$;\\
		Sample $X_i \sim \Pr_{\pi}\left[. | X_{-i} = x_{-i} \right]$ and set $x_{i,t} = X_i$;\\
		For all $j \neq i$, set $x_{j,t} = x_{j,t-1}$;
	}
	\caption{Gibbs Sampling}
\end{algorithm}
\section{Omitted Details from Section~\ref{sec:learnability}}

\label{sec:app-learnability}
%
%

\subsection{Lemmas Required for the Proof of Theorem~\ref{thm:pac-learn}}

\doblemmaone*
\begin{proof}
	. For any $k+1 \le i,j \le n$, such that $i \ne j$, the influence of $i$ on $j$ under the conditional measure is
	\begin{align}
		I_{\pi_{\vec{a}}}(i \to j)  &= \sup_{\substack{(a_{k+1},\ldots,a_n)\backslash (a_i,a_j) \in \Omega^{n-k-2} \\ a_i, a_i' \in \Omega}} \onenorm{\pi(. | \rv{x}_{-i,-j}=a_{-i,-j}, x_i=a_i) - \pi(. |\rv{x}_{-i,-j}=a_{-i,-j},x_i = a_i')} \label{eq:dicd1}\\
		&\le I_{\pi}(i \to j), \label{eq:dicd2}
	\end{align}
	where \eqref{eq:dicd2} holds because the influence of variable $i$ on $j$ is the supremum total variation over conditionings which are Hamming distance one apart and hence upper bounds the partial conditioning we have in \eqref{eq:dicd1}
	Since $\pi$ satisfies the property that for each $i$, $\sum_{j} I_{\pi}(i \to j) \le \alpha$ the above implies that $\sum_{j} I_{\pi_{\vec{a}}}(i \to j) \le \alpha$. Hence the Lemma follows.
\end{proof}

\begin{proof} \textbf{of Lemma \ref{lem:hamm-exp-bound}.}
	We will use the Gibbs sampling algorithm (\ref{alg:gibbs}) for our proof. Since \\$D^{(m)}\left( . |(z_i)_{i=1}^k = (a_i)_{i=1}^k  \right)$ and $D^{(m)}\left( . |(z_i)_{i=1}^k = (a_i')_{i=1}^k  \right)$ satisfy Dobrushin's condition, they have associated Gibbs sampling Markov chains over the state space $(\cX \times \cY)^{m-k}$ which are ergodic. Denote the chains by $M_U$ and $M_V$ respectively. Let $(\rv{U}_t)_{t \ge 0}$ and $(\rv{V}_t)_{t \ge 0}$ be executions of $M_U$ and $M_V$ respectively such that $\rv{\rv{U}}_0 = \rv{V}_0$. We couple these two executions in the following way. At each time step $t$, we choose an index $i \in I_k = \{k+1,\ldots,m \}$ uniformly at random and resample $\rv{U}_{i,t}$ and $\rv{V}_{i,t}$ according to their conditional distributions. And we update $\rv{U}_{i,t}$ and $\rv{V}_{i,t}$ so as to minimize $\Pr[\rv{U}_{i,t} \ne \rv{V}_{i,t}]$. Such a coupling is known as the greedy coupling. We now argue via induction that for all $t$, $\EE\left[d_H(\rv{U}_t,\rv{V}_t) | d_H(\rv{U}_0,\rv{V}_0)=0\right] \le \frac{k\a}{1-\a}$. 
	We have $d_H(\rv{U}_0,\rv{V}_0) = 0$. Assume that for some $t > 0$, $\EE\left[d_H(\rv{U}_t,\rv{V}_t) | d_H(\rv{U}_0,\rv{V}_0)=0\right] \le \frac{k\a}{1-\a}$.
	We will compute a bound on $\EE\left[d_H(\rv{U}_{t+1},\rv{V}_{t+1}) | (\rv{U}_t,\rv{V}_t)\right]$ first.
	We partition the indices $i \in I_k$ into two sets $I_k^{=}$ and $I_k^{\ne}$ as follows.
	\begin{align}
		&I_k^{=} = \left\{ i \in I_k \mid \rv{U}_{i,t} = \rv{V}_{i,t} \right\} \\
		&I_k^{\ne} = \left\{ i \in I_k \mid \rv{U}_{i,t} \ne \rv{V}_{i,t} \right\}
	\end{align}
	To understand whether the Hamming distance goes up or down at time $t+1$, we perform a case analysis.
	Suppose index $i$ was chosen in time step $t+1$ from the set $I_k^{=}$. Then $d_H(\rv{U}_{t+1}.\rv{V}_{t+1}) - d_H(\rv{U}_t,\rv{V}_t) = 1$ or $0$.
	\begin{align}
		&\Pr\left[d_H(\rv{U}_{t+1},\rv{V}_{t+1})-d_H(\rv{U}_t,\rv{V}_t) = 1 \mid \rv{U}_t,\rv{V}_t,i \text{ was chosen for step } t+1 \text{ from } I_k^{=}\right] \notag \\
		&= \Pr\left[ \rv{U}_{i,t+1} \ne \rv{V}_{i,t+1} \mid \rv{U}_t,\rv{V}_t,i \text{ was chosen for step } t+1 \text{ from } I_k^{=}\right] \notag\\
		&\le \norm{D^{(m)}(. | (z_i)_{i=1}^k = (a_i)_{i=1}^k, (z_i)_{i=k+1}^m = \rv{U}_t) - D^{(m)}(. | (z_i)_{i=1}^k = (a_i')_{i=1}^k, (z_i)_{i=k+1}^m = \rv{V}_t)}_{TV} \label{eq:ehb6}\\
		&\le \sum_{j \in [k] \cup I_k^{\ne}} I(j \to i). \label{eq:ehb7}
	\end{align}
	\eqref{eq:ehb6} follows from the definition of Gibbs sampling update probability and the property of total variation distance that it is equal to the worst-case probability of disagreement of a draw from the two distributions over all valid couplings of the two distributions. Hence the greedy coupling should satisfy inequality \eqref{eq:ehb6}. 
	To get \eqref{eq:ehb7}, we use the triangle inequality of total variation distance and bound the total variation in the expression with a sum of total variations between where each term is TV between conditional distributions whose conditioned states have Hamming distance $\le 1$. Each of these total variations is then bounded by their corresponding influence terms (since the influence is defined as a supremum over such conditionings).
	
	Now suppose $i$ was chosen from the set $I_k^{\ne}$ instead. Then $d_H(\rv{U}_{t+1}.\rv{V}_{t+1}) - d_H(\rv{U}_t,\rv{V}_t) = -1$ or $0$.
	\begin{align}
		&\Pr\left[d_H(\rv{U}_{t+1},\rv{V}_{t+1})-d_H(\rv{U}_t,\rv{V}_t) = -1 \mid \rv{U}_t,\rv{V}_t,i \text{ was chosen for step } t+1 \text{ from } I_k^{\ne}\right] \\
		&= \Pr\left[ \rv{U}_{i,t+1} = \rv{V}_{i,t+1} \mid \rv{U}_t,\rv{V}_t,i \text{ was chosen for step } t+1 \text{ from } I_k^{\ne}\right] \\
		&= 1 - \Pr\left[ \rv{U}_{i,t+1} \ne \rv{V}_{i,t+1} \mid \rv{U}_t,\rv{V}_t,i \text{ was chosen for step } t+1 \text{ from } I_k^{\ne}\right] \\
		&\ge 1 - \norm{D^{(m)}(. | (z_i)_{i=1}^k = (a_i)_{i=1}^k, (z_i)_{i=k+1}^m = \rv{U}_t) - D^{(m)}(. | (z_i)_{i=1}^k = (a_i')_{i=1}^k, (z_i)_{i=k+1}^m = \rv{V}_t)}_{TV} \\
		&\ge 1 - \sum_{j \in [k] \cup I_k^{\ne}} I(j \to i), \label{eq:ehb10}
	\end{align}
	using a similar reasoning as above.
	The expected change in the Hamming distance is
	\begin{align}
		&\EE\left[ d_H(\rv{U}_{t+1},\rv{V}_{t+1})-d_H(\rv{U}_t,\rv{V}_t) \mid \rv{U}_t,\rv{V}_t \right] \\
		&~\le \frac{1}{m-k} \EE\left[ \sum_{i \in I_k^{=}} \sum_{j \in [k] \cup I_k^{\ne}} I(j \to i) - \sum_{i \in I_k^{\ne}} \left(1 -  \sum_{j \in [k] \cup I_k^{\ne}} I(j \to i)\right) \mid \rv{U}_t, \rv{V}_t\right] \label{eq:ehb11}\\
		&~\le \frac{1}{m-k}\sum_{j \in [k] \cup I_k^{\ne}} \sum_{i \in [m]\backslash [k]} I(j \to i) - \frac{1}{m-k}d_H(\rv{U}_t,\rv{V}_t) \\
		&\le \frac{(k+d_H(\rv{U}_t,\rv{V}_t))\a - d_H(\rv{U}_t,\rv{V}_t)}{m-k},
	\end{align}
	where \eqref{eq:ehb11} follows from \eqref{eq:ehb7} and \eqref{eq:ehb10}.
	Now,
	\begin{align}
		&\EE\left[ d_H(\rv{U}_{t+1},\rv{V}_{t+1}) | d_H(\rv{U}_0,\rv{V}_0)=0 \right] = \EE\left[ \EE\left[ d_H(\rv{U}_{t+1},\rv{V}_{t+1}) \mid \rv{U}_t,\rv{V}_t\right] \mid d_H(\rv{U}_0,\rv{V}_0) = 0 \right] \\
		&\le \EE\left[ d_H(\rv{U}_t,\rv{V}_t) + \frac{(k+d_H(\rv{U}_t,\rv{V}_t))\a - d_H(\rv{U}_t,\rv{V}_t)}{m-k} \mid d_H(\rv{U}_0,\rv{V}_0) = 0 \right] \\
		&\le \frac{k\a}{1-\a}\left(1 - \frac{1-\a}{m-k} \right) + \frac{k \a }{m-k} = \frac{k\a}{1-\a}. \label{eq:ehb14}
	\end{align}
	Hence we have shown that $\EE\left[ d_H(\rv{U}_{t+1},\rv{V}_{t+1}) | d_H(\rv{U}_0,\rv{V}_0)=0 \right] \le \frac{k\a}{1-\a}$. Now, we have that 
	\begin{align}
		\EE[d_H(\rv{U},\rv{V})] &= \lim_{t \to \infty} \EE[d_H(\rv{U}_t,\rv{V}_t) \mid \rv{U}_0,\rv{V}_0] \label{eq:ehb15}\\
		&\le  \frac{k\a}{1-\a}, \label{eq:ehb16}
	\end{align}
	where \eqref{eq:ehb15} follows because the Gibbs sampler we consider is ergodic and \eqref{eq:ehb16} follows from \eqref{eq:ehb14}.
\end{proof}

\end{document}